\documentclass[AMA,STIX1COL]{WileyNJD-v2}

\articletype{Article Type}%

\usepackage{changes,color,graphicx,multirow,caption,float}

\newtheorem{thm}{Theorem}

\DeclareMathOperator*{\argmin}{arg\,min}

\newcommand{\ignore}[1]{}{}
\newcommand{\vertiii}[1]{{\left\vert\kern-0.25ex\left\vert\kern-0.25ex\left\vert #1 
    \right\vert\kern-0.25ex\right\vert\kern-0.25ex\right\vert}}

\begin{document}

\title{Varying-coefficients for regional quantile  via KNN-based LASSO with applications to health outcome study}

\author[1]{Seyoung Park}

\author[1]{Eun Ryung Lee}

\author[2]{Hyokyoung G. Hong*}

\authormark{Park \textsc{et al}}

\address[1]{\orgdiv{Department of Statistics}, \orgname{Sungkyunkwan University}, \orgaddress{\state{Seoul}, \country{Republic of Korea}}}

\address[2]{\orgdiv{Biostatistics Branch, Division of Cancer Epidemiology and Genetics}, \orgname{NCI/NIH}, \orgaddress{\state{MD}, \country{United States of America}}}

 \corres{*Hyokyoung G. Hong,	 \email{grace.hong@nih.gov}}


\abstract{
Health outcomes, such as body mass index and cholesterol levels, are known to be dependent on age and exhibit varying effects with their associated risk factors. In this paper, we propose a novel framework for dynamic modeling of the associations between health outcomes and risk factors using varying-coefficients (VC) regional quantile regression via K-nearest neighbors (KNN) fused Lasso, which captures the time-varying effects of age. The proposed method has strong theoretical properties, including a tight estimation error bound and the ability to detect exact clustered patterns under certain regularity conditions. To efficiently solve the resulting optimization problem, we develop an alternating direction method of multipliers (ADMM) algorithm. Our empirical results demonstrate the efficacy of the proposed method in capturing the complex age-dependent associations between health outcomes and their risk factors.
}

\keywords{KNN,  Lasso, Varying-coefficients,  Health outcome study, Regional quantile regression}

\jnlcitation{\cname{%
\author{S. Park}, 
\author{E. R. Lee},   and
\author{H. G. Hong} (\cyear{2022}), 
\ctitle{Regional quantile modeling of varying-coefficients with applications to health disparity study}}}

\maketitle

\footnotetext{Seyoung Park and Eun Ryung Lee are co-first authors}

\section{Introduction}



In recent years, health outcome research has played a crucial role in identifying disparities among different racial and ethnic groups, enabling policymakers and clinicians to make informed decisions for individuals from diverse socioeconomic backgrounds. Numerous studies have evaluated racial disparities in various health outcomes, such as body mass index (BMI), sleep duration, and cholesterol levels. \cite{beydoun2009gender, grandner2016sleep, stewart2002racial}  BMI, in particular, has been widely used as a health risk indicator in clinical and public health research. \citep{SRo2002, NFQ2015} To better understand the factors that drive BMI, researchers from various fields have explored statistical techniques to identify important predictors. For example, Huang et al. \cite{HMXZ2009} proposed a group bridge approach for selecting risk factors of BMI, while Rehkopf et al. \cite{RLS2011} used a random forest technique to rank risk factors according to their relative importance score. Gao et al. \cite{GPRZ2017} proposed a variable selection method to identify relevant BMI risk factors, assuming that the impact of these determinants is an unknown function of categorical demographic variables.

Existing literature on BMI studies has primarily focused on modeling the relationships between risk factors and the mean BMI level from average individuals, while ignoring age or time dependency. This approach has limited insights regarding the BMI distribution in the population. In this paper, we propose a tailored statistical approach for detecting dynamic and heterogeneous associations between health outcomes and risk factors. Specifically, we adopt a varying-coefficient quantile model framework to explain associations between health outcomes and risk factors in the presence of varying effects of age.

Varying-coefficient (VC) models have gained popularity in both theoretical and practical aspects since their inception. \citep{JFWZ1999, TJHT1993, QLDO2013, QLJR2010, HWYX2009, LLCGS19912}  To deal with high-dimensional variables, VC approaches have incorporated variable selection procedures. \citep{SMPXKS2015, LWHLH2008, HWYX2009, GPRZ2017, MOK2007} For instance, Gao et al. \cite{GPRZ2017} considered the variable selection problem for the categorical varying-coefficient model, based on a penalized approach using group Lasso. \citep{MYYL:2006} Nonparametric approaches, based on basis functions, have also been widely used for estimation and variable selection in VC models. \citep{LWHLH2008, WJL:2011, LXAQ:2012, OKMP:2015, Honda2019}

We consider a  VC regional quantile model in this paper, which is defined as follows:
\[Y(T,X)=X^\top \beta(T, \tau)+\epsilon(\tau)\quad \mbox{for}\quad \tau \in \Delta, \]
where $P(\epsilon(\tau) \leq 0 | X,T) = \tau$ for $\tau \in \Delta$, $\Delta \subset (0,1)$ is a quantile interval of interest, $T$ is an age index, $X$ are covariates, and $Y$ is an outcome variable. This model framework naturally arises in many real-life applications. \citep{SPEL2021} Analyzing the behavior of VC over a range of quantiles is important in the field of regression analysis. When various quantile levels are of interest, a typical approach is to individually fit a quantile regression and obtain inference at each quantile level, which may result in a loss of estimation efficiency because regressions at adjacent quantile levels are expected to share similar features. In such cases, a regional quantile regression approach \citep{Zheng2015globally, SPEL2021} can be a useful alternative and may lead to a more efficient estimation procedure.

Despite the importance of both theoretical and practical aspects, there is a lack of literature on the selection and estimation of clustered patterns in the coefficient function under VC regional quantile settings. Our main objective is to detect regional clustered patterns in the regional quantile regression coefficients, $\beta(T, \tau)$, using K-nearest neighbors (KNN) fused Lasso. The proposed method can identify de-noised clustered patterns between the risk factors and health outcomes, select important determinants of the regional health outcome quantiles (such as the upper level of the BMI distribution), and simultaneously estimate varying coefficients across both age and quantiles of BMI.

Our work is related to Padilla et al., \cite{PSCW2020} who combined fused Lasso with the KNN procedure in a mean-based regression model, and more recently, Ye and Padilla, \cite{SSYOHMP2021} who proposed a nonparametric quantile regression using KNN-based fused Lasso penalty. However, it is important to note that Ye and Padilla \cite{SSYOHMP2021} modeled the conditional quantile of the response variable as a nonparametric function of covariates at a fixed quantile level. Li and Sange \citep{FLHS2019} considered the spatially clustered varying-coefficient model using the fused Lasso method, but it was based on a linear model and only allowed a simple tree structure in the graph. Yang et al. \cite{YCC2017} adopted the parametric quantile regression method of Frumento and Bottai \cite{PFMB2016} to analyze longitudinal BMI data. These approaches differ from our proposed method, which considers a structured nonparametric model for regional quantiles. The proposed model shares common structural information across adjacent quantile levels and may better detect patterns over both quantiles, $\tau$, and index,  $T$. Our work is also more general since we do not require parametric specifications for quantile coefficient functions, which is not a trivial assumption.

We summarize the key properties of the proposed method as follows:
(1) The proposed approach based on quantile regression yields robust estimates despite the violation of normality assumption;
(2) The conditional quantile framework allows us to explore the heteroscedastic relationship between different sublevels of the dependent variable and covariates, which cannot be captured by the standard regression approach;
(3) By adopting a nonparametric approach to modeling the covariate effects under the regional quantile VC framework, our method can handle a nonlinear relationship between the dependent variable and covariates;
(4) Compared to its local counterpart, the regional quantile approach provides more stable and interpretable results;
(5) By leveraging the insight from the quantile KNN fused Lasso, \cite{SSYOHMP2021} our algorithm can detect underlying clustered patterns in the VC functions;
and (6) The proposed optimization via the efficient alternating direction method of multipliers (ADMM) algorithm is computationally scalable since each updating step has a closed-form solution and utilizes parallel computing.

The remainder of this paper is organized as follows. In Section 2, we present the proposed VC quantile model via KNN fused Lasso method, its theoretical properties, and the ADMM algorithm. In Section 3, we evaluate the finite sample performance of the proposed methods using simulation studies. In Section 4, we apply the proposed method to two health outcome studies. Finally, Section 5 concludes the paper and discusses potential future research questions. Technical proofs, additional simulation results, and figures are provided in the Supplementary material.

\section{Varying-Coefficient Quantile model via KNN Fused Lasso} \label{sec2vc}
In this section, we propose VC regional quantile regression via KNN fused Lasso and {study} theoretical properties and   ADMM algorithm.
In the VC regional quantile model,  for units $i=1,\ldots, n$ with a covariate vector $x_i = [x_{i1},\ldots, x_{ip}]^\top \in \mathbb{R}^p$,
a response, $y_i$, can be modeled as
\[
y_i = \sum_{j=1}^p x_{ij} \beta^o_j(t_i,\tau) + \epsilon_i(\tau) \quad \mbox{for}\quad \tau \in \Delta,
\]

where {$t_i \in [0,1]$ is a time  index variable, e.g., age in our applications,}  $\beta^o_j(t,\tau)$  is the underlying coefficient function for the $j$th covariate, $\Delta \subset (0,1)$ is the quantile interval of interest, and the conditional $\tau$-th quantile of a random error $\epsilon_i(\tau)$ given $(t_i, x_i)$ is zero.

We first construct the KNN graph $G$ based on $(t_i, \tau_{i}), ~1\le i \le n$ in the domain of coefficient functions $\beta_j$'s, where 
the quantile level $\tau_i$'s are randomly chosen from $\Delta$, { i.e., $\tau_i \sim \mbox{uniform}(\Delta)$, which is to discretize $\Delta$.} Specifically, each $(t_i, \tau_{i})$, for $i=1,\ldots, n$, corresponds to a node in the graph $G$ and its edge set $E$ contains the pair $(i,k)$ for $i \neq k$, if and only if  $(t_i, \tau_{i})$ is among the K-nearest neighbors of $(t_k, \tau_k)$.
We propose the regional quantile KNN fused Lasso (RQF) method in varying-coefficients models 
 to estimate the coefficient function  $\beta^o_j(t,\tau)$ as follows:
\begin{equation} \label{eqmain}
\min_{\beta_1,\ldots, \beta_p} \frac{1}{n}\sum_{i=1}^n \rho_{\tau_i}\left(y_i -  \sum_{j=1}^p x_{ij} \beta_{j}(t_i, \tau_i)\right)  +  \lambda \sum_{j=1}^p  \|H \beta_j\|_1,
\end{equation}
where $\beta_j = (\beta_{j1},\ldots, \beta_{jn})^\top \in \mathbb{R}^n$ with $\beta_{ji} :=  \beta_{j}(t_i, \tau_i)$ for $i=1,\ldots, n$.

Here, $H$ is an $|E| \times n$ oriented incidence matrix of the KNN
graph $G$, and thus each row of $H$ corresponds to an edge $e \in E$. \cite{PSCW2020} 
Specifically, if the $m$-th edge in $G$ connects the $i_m$-th and $k_m$-th nodes, then

\[
H_{m,l} =
\begin{cases}
1 & \mbox{if}\ l=i_m \\
-1 & \mbox{if}\ l=k_m\\
0 & \mbox{otherwise},
\end{cases}
\]
and  $(H \beta_j)_m = \beta_{j,i_m} - \beta_{j,k_m}$, $\|H \beta_j\|_1 =  \sum_m |\beta_{j,i_m} - \beta_{j,k_m}|$.
In \eqref{eqmain}, we considered a single quantile level $\tau_i$ for each sample $i$ to reduce  computational cost. If we consider fixed multiple quantile levels $\tau_1,\ldots, \tau_K \in \Delta$,  as in the composite quantile regression,  
 then the computation is nearly infeasible. {In addition, the large sample size as in our real data examples would make it worse}.

In \eqref{eqmain}, the fused Lasso  penalty enforces sparsity of the difference
in two edge-connected coefficients. This allows the estimation of  coefficients with clustered patterns if edge sets are selected appropriately.
Using the obtained $\hat{\beta}_j$, we can estimate the value of coefficient corresponding to a new $(t, \tau)$ by the averaged estimated {values of the KNN} as follows:
\begin{equation}\label{eqneighbor}
\hat{\beta}_j(t,\tau) = \frac{1}{K} \sum_{i=1}^n \hat{\beta}_j(t_i,\tau_i) \ 1\{(t_i, \tau_i) \in N_K(t,\tau) \},
\end{equation}
where $N_K(t,\tau)$ is the set of {KNN} of $(t,\tau)$ in a training data $\{(t_i, \tau_i):\ i=1,\ldots, n\}$.  Thus, it leads to smooth and locally adaptive VC estimates.

Note that the tuning parameter $\lambda$ in (1) controls the number of clusters in regression coefficients. When $\lambda=0$, it reduces to the ordinary regional quantile regression; {when} $\lambda \to \infty$, RQF yields a { nearly} constant regression coefficient { in that $\hat{\beta}_j(t_i,\tau_i) = \hat{\beta}_j(t_k,\tau_k)$ for $(i,k) \in E$.}
With an appropriate choice of $\lambda$, RQF produces clustered regression coefficients. 
In practice, we propose the following BIC to choose $\lambda$:
\[
\mbox{BIC} = \log\left\{\sum_{i=1}^n \rho_{\tau_i}\left(y_i -  \sum_{j=1}^p x_{ij} \beta_{j}(t_i, \tau_i)\right)\right\} +  \frac{\log n}{n} \sum_{j=1}^p \|H \beta_j\|_0,
\]
where $\|H \beta_j\|_0$ represents  the number of  nonzero values in $H\beta_j$.

\subsection{Notations}

Throughout the paper,    
 $\|A\|_{op}$ represents an operator norm of a matrix $A$, i.e., the maximum singular value of $A$, and
 $\lambda_{\max}(B)$ is a maximum eigenvalue of a symmetric matrix $B$.
We write $a \lesssim b$ if $a \le C_1 b$ for some
positive  constant $C_1$, $a \asymp b$ if $a \lesssim b$ and $b \lesssim a$,
 and  $a \vee b$  and $a \wedge b$ to
denote $\max(a,b)$ and $\min(a,b)$, respectively.
 For a vector $x$, let $\mbox{support}(x) = \{j:\ x_j \neq 0\}$ be the index set
of non-zero entries of the vector $x$ and $\|x\|_0 = |\mbox{support}(x)|$ is the cardinality of $\mbox{support}(x)$.
For a vector $x$ and the index set $S$, let $x_S$  be the subvector of $x$ with components in $S$. 
For a matrix $M=(M_{ij})$, let $\|M\|_F=\left(\sum_i \sum_j M_{ij}^2\right)^{1/2}$, $\|M\|_1=\sum_i \sum_j |M_{ij}|$,  $\|M\|_{\max} = \max_{i,j}|M_{ij}|$, 
$\vertiii{M}_1 = \max_j \sum_i |M_{ij}|$, and $\vertiii{M}_{\infty} = \max_i \sum_j |M_{ij}|$.
For a random sample $t_1,\ldots, t_n$, let $E_n[t_i] := n^{-1} \sum_{i=1}^n t_i$.

\subsection{Theoretical properties}
\label{sec:theory:vc}
{For easier presentation, we introduce a new parameter $\theta$, which is a reparametrization of $\beta_j(t_i, \tau_i)$ for $i=1,\ldots, n$ and $j=1,\ldots, p$.} Suppose that in the KNN graph $G$, there exists $L \ge 1$ connected components, say
$G_1,\ldots, G_L$, { where the subgraph $G_l$ has a node set $V_l$ and  an edge set $E_l$ with $|V_l| = n_l$, $\bigcup_{l=1}^L V_l = \{1,\ldots, n\}$, and $V_l \cap V_{\tilde l} = \emptyset$ for $l \neq \tilde l$. Note that given $K$ and $(t_i, \tau_i), ~i=1,\ldots,n$,  the number of connected components $L$ and the graph $G$ are determined.  
For example, if $n$ is a square number and $(t_i,\tau_i),~i=1,\ldots,n$, are $\sqrt{n} \times \sqrt{n}$ rectangular grid points in $[0,1]\times \Delta$, then the KNN graph $G$ with $K\ge 4$ is itself connected, i.e., $L=1$.   
By rearranging sample indices}, let  $V_l$'s be increasingly ordered sets, i.e., $V_1 = \{1,\ldots, n_1\}, V_2 = \{n_1+1,\ldots, n_1+n_2\}, \ldots,
V_L = \{\sum_{l=1}^L n_l - n_L +1,\ldots, \sum_{l=1}^L n_l\}$.
We can write $H= \mbox{Block}(H_1, \ldots, H_L) \in \mathbb{R}^{|E| \times n}$ as a block diagonal matrix consisting of $H_l$'s,
where rows of $H_l$ corresponds to edges of the $l$-th connected group  $G_l$ such that $H_l \in \mathbb{R}^{|E_l| \times n_l}$. Define $\tilde{H}_l = [1_{n_l}/\sqrt{n_l}, H_l^\top]^\top \in \mathbb{R}^{(|E_l| + 1) \times n_l}$, where $1_{n_l}$ is $n_l$-dimensional vector with all $1$'s.
We can see that $\tilde{H}_l^\top \tilde{H}_l = H_l^\top H_l + \frac{1}{n_l} 1_{n_l} 1_{n_l}^\top$, where $H_l^\top H_l$ represents the Laplacian matrix, or called the graph Laplacian  of the component $G_l$.  Thus, $\tilde{H}_l^\top \tilde{H}_l \in \mathbb{R}^{n_l \times n_l}$ is invertible.

Without loss of generality, we write 
\[
\tilde{H}_l
= \begin{bmatrix}
\tilde{H}_l^{(1)} \\
\tilde{H}_l^{(2)}
\end{bmatrix},
\quad \mbox{where}\quad \tilde{H}_l^{(1)} \in \mathbb{R}^{n_l \times n_l} \quad \mbox{is invertible and}\quad \tilde{H}_l^{(2)} \in \mathbb{R}^{(|E_l| + 1-n_l) \times n_l}.
\]
This can be achieved  by rearranging the rows of $\tilde{H}_l$ such that the first $n_l$ rows corresponds to the edges of the minimum spanning tree of $G_l$ and the vector $1^\top_{n_l}/\sqrt{n_l}$.
We can write the parameter $\beta_j$ with $\beta_j = [(\beta^j)^\top_{G_1}, \ldots, (\beta^j)^\top_{G_L}]^\top$, where $\beta^j_{G_l} =(\beta_j(t_i,\tau_i):i\in V_l)^\top \in \mathbb{R}^{n_l}$ is a subvector of $\beta^j$ corresponding to the node (index) in the graph $G_l$.
Then, $\beta^j_{G_l}$ can be rewritten using a new parameter $\theta^j_{G_l} \in \mathbb{R}^{n_l}$ as $\beta^j_{G_l} = [\tilde H_l^{(1)}]^{-1} \theta^j_{G_l}$.   
Let $\theta_j = [(\theta^j)^\top_{G_1}, \ldots, (\theta^j)^\top_{G_L}]^\top \in \mathbb{R}^{n }$ for $j=1,\ldots, p$, and $\theta = [\theta_1^\top,\ldots, \theta_p^\top]^\top \in  \mathbb{R}^{np }$.

Then, the problem \eqref{eqmain} can be rewritten as
\begin{equation}\label{eqmainprob}
\min_{\theta \in \mathbb{R}^{np}} \frac{1}{n}\sum_{i=1}^n \rho_{\tau_i}\left(y_i - \tilde{x}_i^\top \theta \right)  +  \lambda \sum_{j=1}^p \|\theta_{j,B_1}\|_1 +
\lambda \sum_{j=1}^p \|\ddot{H}\theta_j\|_1,
\end{equation}
 where $\ddot{H}$ is defined in  Section A of the Supplementary material,  $\theta_{j,B_1} = [(\theta^j_{G_1, B_1})^\top, \ldots, (\theta^j_{G_L, B_1})^\top]^\top \in \mathbb{R}^{n-L}$, and $\theta^j_{G_l, B_1}$'s are defined in the Supplementary material.
From the estimate $\hat \theta$ of \eqref{eqmainprob}, we can obtain the estimate $\hat \beta_j  = [(\hat\beta_{G_1}^j)^\top, \ldots, (\hat\beta^j_{G_L})^\top]^\top$, where $\hat\beta^j_{G_l} = [\tilde H_l^{(1)}]^{-1} \hat \theta^j_{G_l}$.

 Let  $\theta^o =[(\theta_1^o)^\top,\ldots, (\theta_p^o)^\top]^\top\in \mathbb{R}^{np},$
$\theta^o_j\in \mathbb{R}^{n}$, and 
$\theta_{j,B_1}^o\in \mathbb{R}^{n-L}$ 
be the underlying vectors similarly defined as for $\theta$, $\theta_{j}$, and $\theta_{j,B_1}$, respectively, which is the function of $\beta^o_j( t_i, \tau_i)$'s. 
Let $S_{1j} :=  \mbox{support}(\theta^o_{j,B_1}) =  \{l: [\theta^o_{j,B_1}]_l \neq 0\}$,  $S_{2j} := \mbox{support}(\ddot{H} \theta^o_j) =   \{l: [\ddot{H} \theta^o_j]_l  \neq 0\}$, 
$s_{1j} = |S_{1j}|$,  $s_{2j} = |S_{2j}|$, $s_1 := \sum_j s_{1j}$, and $s_2 := \sum_j s_{2j}$. The underlying clustered pattern for each $j$ is explained by the support set $S_j:=S_{1j}\bigcup S_{2j}$, indicating which edge-connected differences in $\beta^o_j$ are nonzero.  For example, if $S_j = \emptyset$, then  $\beta^o_j(t_i,\tau_i) = \beta^o_j(t_k,\tau_k)$ for $(i,k) \in E$.
Let $S_1^U=\bigcup_j S_{1j}$, $S_2^U=\bigcup_j S_{2j}$,  $S:= S_{1}^U \cup S_{2}^U$, and $A := S_{1}^U \cup B_2$, where  $B_2$ is the index set corresponding to the vector $1_{n_l}/\sqrt{n_l}$. 
Let $s^{(j)} = |S_j|$ and $s=\sum_{j=1}^p s^{(j)} = s_1 + s_2$. We estimate $S_j$ by $\hat S_j = \mbox{support}(H\hat\beta_j)$ and $S$ by $\hat{S}=\bigcup \hat S_j.$ To facilitate theoretical analyses, we impose the following conditions.
\begin{assumption}
\label{cond1}
For any $\tau\in \Delta$ and $1\le i\le n$, 
the conditional density function of the random error $\epsilon_i(\tau)$ at $\tau$th quantile level, i.e., {$f_i^{(\tau)}(x)$} has a continuous derivative $(f_i^{(\tau)})'(x)$, and satisfies 
$
\sup_i\sup_x f_i^{(\tau)}(x) \le \bar{f} \quad \mathrm{and} \quad \sup_i|(f_i^{(\tau)})'(x)| \le \bar{f}' \quad  \mathrm{for}\quad |x| \le c_1
$
for some positive constants $c_1$, $\bar{f}$, and $\bar{f}'$. Moreover,
$
\inf_i f_i^{(\tau)}(0) \ge \underbar{f}
$
for some positive constant $\underbar{f}$.
\end{assumption}

\begin{assumption} 
\label{cond2}
Define the following restricted set:
\[
C(A, S_2^U) := \left\{\delta =(\delta_1^\top, \ldots, \delta_p^\top )^\top\in \mathbb{R}^{np}: \delta_j\in \mathbb{R}^n, \; \|\delta_{A^c}\|_1 + \sum_{j=1}^p \|(\ddot{H}\delta_j)_{S_{2j}^c}\|_1 \le 3\left(\|\delta_{A}\|_1 + \sum_{j=1}^p \|(\ddot{H}\delta_j)_{S_{2j}}\|_1 \right)\right\}.
\]
It holds that the design matrix $\tilde X$ satisfies
\[
\kappa=  \inf_{\delta \in C(A, S_2^U),\ \delta  \neq 0}  \frac{E_n[(\tilde x_i^\top \delta)^2]}{\|\delta\|^2} > 0, \quad
q :=\frac{3 \underbar{f}^{3/2}}{8 \bar{f}'}  \inf_{\delta \in C(A, S_2^U),\ \delta  \neq 0}  \frac{E_n[(\tilde x_i^\top \delta)^2]^{3/2}}{E_n[|\tilde x_i^\top \delta|^3]} > 0.
\]
\end{assumption}

\begin{assumption}
\label{cond3}
The minimum nonzero signal difference in $\beta^o_j$ is greater than the order of $O_p\left(\sqrt{s \log(np)/n} \right)$, i.e.,
$\min_{l \in S} | [H \beta^o_j]_l| \gtrsim \sqrt{s \log(np)/n}$.
\end{assumption}

{
\begin{assumption}
\label{cond4}
Let
$Q(\theta) = E\left[\frac{1}{n}\sum_{i=1}^n \rho_{\tau_i}\left(y_i - \tilde{x}_i^\top \theta \right)\right]$
and $\bar M= \int_0^1 \nabla^2 Q(\theta+ t(\hat{\theta}-\theta)) dt \in \mathbb{R}^{np \times np}$. 
We assume $ \vertiii{\bar M_{(S_1^U)^c,S_1^U}\bar  M_{ S_1^U,S_1^U}^{-1}}_1 \le 1$ and  $\vertiii{\dddot{H}_2}_{1} \vee \vertiii{\ddot{T}}_\infty \le 1$, where 
$\dddot{H}_2$ and $\ddot{T}$ are defined in the proof of Theorem \ref{mainthm2} in Section C of the Supplementary material.
\end{assumption}}

{ Assumption \ref{cond1} is a common assumption used in quantile regression literature, \cite{Belloni2011ell1, Zheng2015globally, SPXH2017} which} imposes only mild assumptions on the conditional
density of the response variable given covariates,  not imposing any normality or homoscedasticity assumptions. 
The first part in Assumption \ref{cond2} is the restricted eigenvalue (RE) condition, which is analogous to the assumptions in the existing literature. \cite{Belloni2011ell1,BRT2009,ECTT2007}
The second part in Assumption \ref{cond2} is the restricted nonlinear impact (RNI) condition, \cite{Belloni2011ell1,Zheng2015globally}  which controls the quality of minoration of the quantile
regression objective function by a quadratic function over the restricted set.
Assumption \ref{cond3} is a beta-min type condition, which imposes a lower bound of the nonzero signal differences. 
Assumption \ref{cond4} is a irrepresentable type condition, \cite{MB06,ZY06,Zheng2015globally} which restricts  correlations among covariates.

The mean squared error (MSE) of $\hat\theta$ is defined as $\|\hat\theta - \theta^o \|_n^2 := {n}^{-1} \|\hat{\theta} - \theta^o\|^2$. The following theorem presents upper bound of the MSE of $\hat\theta$. All the technical proofs are deferred to  Section C of the Supplementary material.

\begin{thm} \label{thmmain1}
Suppose that Assumptions \ref{cond1}-\ref{cond2} hold.  If $\lambda \asymp \sqrt{\log(pn)/n}$, then
 $\hat{\theta}$ satisfies
$
\|\hat\theta - \theta^o \|_n^2   \lesssim  {s  \log(pn)}/{n^2}.
$
\end{thm}
Theorem \ref{thmmain1} implies that the MSE of $\hat\theta$ decreases asymptotically to zero as $n \to \infty$ assuming that $s$ grows with a rate as $o(n^2/ \log (pn))$. Note that such a growth rate for $s$ is satisfied when $|\{j \in \{1,\ldots, p\}: s^{(j)} \neq 0\}| = o(n/\log(pn))$, e.g., many $\beta^o_j$ are constant functions.
{ Because $\|\hat\beta_j - \beta_j\|^2 =(\hat\theta_j - \theta_j)^\top  { \bar{H}^\top \bar{H}}
(\hat\theta_j - \theta_j)$,  where $\bar H$ is defined in  Section A of the Supplementary material, Theorem  \ref{thmmain1} also gives the MSE of $\hat \beta_j$'s as follows:
\[
\sum_{j=1}^p \|\hat\beta_j - \beta_j\|_n^2 := \sum_{j=1}^p  \|\hat\beta_j - \beta_j\|^2/n \le \lambda_{\max}\left( { \bar{H}^\top \bar{H}}\right) \sum_{j=1}^p \|\hat\theta_j - \theta_j\|^2 /n
\lesssim  \lambda_{\max}\left( { \bar{H}^\top \bar{H}}\right) {s  \log(pn)}/{n^2}.
\]
}
{ 
If $\lambda_{\max}\left(  \bar{H}^\top \bar{H}\right)=O_p(n)$, the MSE rate is bounded by $s \log (pn)/{n}$, which is within logarithmic factors of the oracle rate that can be obtained with known $S$.
}


The following theorem shows that RQF detects the underlying true set $S$.
\begin{thm}\label{mainthm2}
Suppose that Assumptions \ref{cond1}-\ref{cond4} holds.  If $\lambda \asymp \sqrt{\log(pn)/n}$ and $s_1 \lesssim n \log (pn)$, then
\[
P( \hat{S} = S) \to 1.
\]
\end{thm}

Theorem \ref{mainthm2} implies that RQF can detect the underlying subclusters
{  
in the graph $G$ constructed from the points $(t_i, \tau_i),\; i=1,\ldots, n$
as long as any nodes in the same underlying subcluster are connected in the graph $G$.
}



\subsection{ADMM algorithm}
The optimization of \eqref{eqmain} can be computed using the ADMM algorithm. 
Let $\beta_{ji} = \beta_{j}(t_i, \tau_i)$, $\beta_j = [\beta_{j1},\ldots, \beta_{jn}] \in \mathbb{R}^n$, and $\beta= [\beta_1^\top,\ldots, \beta_p^\top]^\top \in \mathbb{R}^{np}$.
Then, we can rewrite the optimization as follows by introducing supplementary
variables $z_j = [z_{j1}, \ldots, z_{jn}]^\top \in \mathbb{R}^n$ and $z = [z_1^\top,\ldots, z_p^\top]^\top \in \mathbb{R}^{np}$,
 \[
\min_{\beta, z \in \mathbb{R}^{np}} \sum_{i=1}^n \rho_{\tau_i}\left(y_i -  \sum_{j=1}^p x_{ij} \beta_{ji} \right)  + \lambda  \sum_{j=1}^p \|H z_j\|_1, \ \mbox{where}\ z_j = \beta_j \quad \mbox{for}\quad j=1,\ldots, p.
\]

By the augmented Lagrangian method, we consider the following
\begin{eqnarray}
 \sum_{i=1}^n \rho_{\tau_i}\left(y_i -  \sum_{j=1}^p x_{ij} \beta_{ji}\right) + \lambda  \sum_{j=1}^p \|H z_j\|_1
+ \frac{\eta}{2} \sum_j \|\beta_j-z_j+u_j\|^2, \label{eqmainopt}
\end{eqnarray}
where $u_j = [u_{j1},\ldots, u_{jn}]^\top \in \mathbb{R}^n$ are the dual variables and $\eta >0$ is a step-size parameter. 
{In \eqref{eqmainopt}, we need to update $\beta$, $z$, and $u_j$'s. Updating $\beta$ and $z$ requires some mathematical derivations, but dual variables $u_j$ will be simply updated according to the updated $\beta$ and $z$.}
Let $F(\beta, z, u)$ be the objective function of  \eqref{eqmainopt}.
We iteratively solve \eqref{eqmainopt} as follows:
\begin{eqnarray}
\beta^{(t)} &=& \argmin_{\beta}  F(\beta,z^{(t-1)},u^{(t-1)}) \label{sup:solve:1}\\
z^{(t)} &=& \argmin_{z}  F(\beta^{(t)},z,u^{(t-1)}) \label{sup:solve:2}\\
u_j^{(t)} &=&  u_j^{(t-1)}+  \eta(\beta_j^{(t)} - z_j^{(t)})  \quad \mbox{for}\quad j=1,\ldots, p.\nonumber
\end{eqnarray}
   
For each step, we omit the superscript notations $(t)$ and $(t-1)$ 
whenever it does not cause any confusion.

\vspace{1cm}

\noindent
\textbf{Update $\beta$}\\
{For each $i=1,\ldots, n$, define $\beta^i := [\beta_{1i}, \ldots, \beta_{pi}]^\top$, $z^i := [z_{1i}, \ldots, z_{pi}]^\top$ and $u^i := [u_{1i}, \ldots, u_{pi}]^\top$ as  $p$-dimensional vectors.
For each $i=1,\ldots, n$, solving \eqref{sup:solve:1} is equivalent to solving the following:}
\[
\min_{\beta^i} \rho_{\tau_i}(y_i - x_i^\top \beta^i) + \frac{\eta}{2} \|\beta^i - z^i + u^i\|^2.
\]
By the Karush-Kuhn-Tucker (KKT) conditions, minimizer $\hat{\beta}^i$ must satisfy
\[
-v_i x_i + \eta(\hat{\beta}^i - z^i + u^i) = 0\ \mbox{ or } \  \hat{\beta}^i = (z^i - u^i) + v_i x_i / \eta,
\]
where 
\[
v_i  =
\begin{cases}
\tau_i - 1 & \mbox{if}\ y_i < x_i^\top \hat{\beta}^i \\
\tau_i  & \mbox{if}\ y_i > x_i^\top \hat{\beta}^i \end{cases}
\]
\[
v_i \in [\tau_i-1, \tau_i]  \mbox{~if}\  y_i =x_i^\top \hat{\beta}^i. 
\]
Suppose that $y_i =x_i^\top \hat{\beta}^i$.  Then, it must hold that
\[
-\tau_i  x_i^\top x_i / \eta  \le  x_i^\top(z^i - u^i) - y_i = - v_i x_i^\top x_i / \eta \le (1-\tau_i)  x_i^\top x_i / \eta,
\]
which implies that  if $x_i^\top(z^i - u^i) - y_i  \in [-\tau_i  x_i^\top x_i/\eta, (1-\tau_i)  x_i^\top x_i / \eta]$, then
 $\hat{\beta}^i = (z^i - u^i) + \hat v_i x_i / \eta$, where $\hat v_i = -\eta (x_i^\top(z^i - u^i) - y_i) / (x_i^\top x_i)$.
Similarly, we can consider the remaining cases and  obtain the following updates:
\[
\hat\beta^i = 
\begin{cases}
(z^i - u^i) + (\tau_i-1) x_i / \eta & \mbox{if}\   x_i^\top(z^i - u^i) - y_i  >  (1-\tau_i)  x_i^\top x_i/\eta\\
 (z^i - u^i) + \tau_i x_i / \eta  & \mbox{if}\   x_i^\top(z^i - u^i) - y_i  < -\tau_i  x_i^\top x_i/\eta\\
(z^i - u^i) + \hat v_i x_i / \eta & \mbox{else}.
\end{cases}
\]
This can be efficiently computed via a parallel implementation.

\vspace{1cm}

\noindent
\textbf{Update $z$}\\
For each $j=1,\ldots, p$, recall that $z_j = [z_{1j}, \ldots, z_{nj}]^\top$.
{Then, for each $j=1,\ldots, p$, solving \eqref{sup:solve:2} is equivalent to solving the following:}
\[
\min_{z_j} \frac{\eta}{2} \|z_j - \beta_j - u_j\|^2 + \lambda \|H z_j\|_1,
\]
{ which corresponds to the KNN-fused Lasso \cite{PSCW2020} and can be solved by the parametric max-flow algorithm.  \cite{ACJD2009,PSCW2020, SSYOHMP2021}}
This also allows parallel implementation.\\

 \begin{figure}[t] \label{figrplot3true}
 \centering
\includegraphics[width=0.7\textwidth, height=7cm]{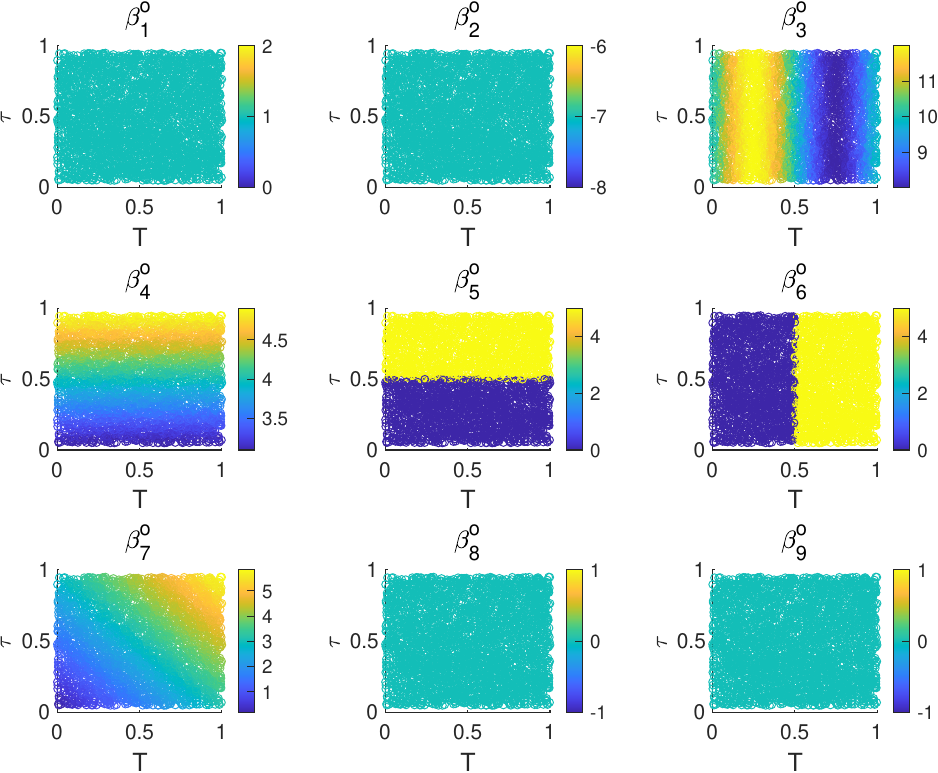}\\
\caption{Underlying coefficient function $\beta_j(t,\tau)$ for $j=1,\ldots, 9$. The x-axis and y-axis represent the index $T \in (0,1)$ and the quantile level $\tau \in (0.05, 0.95)$, respectively.}
\end{figure}  

 \begin{figure}[b] \label{figrplot3}
\centering
\includegraphics[width=0.7\textwidth, height=7cm]{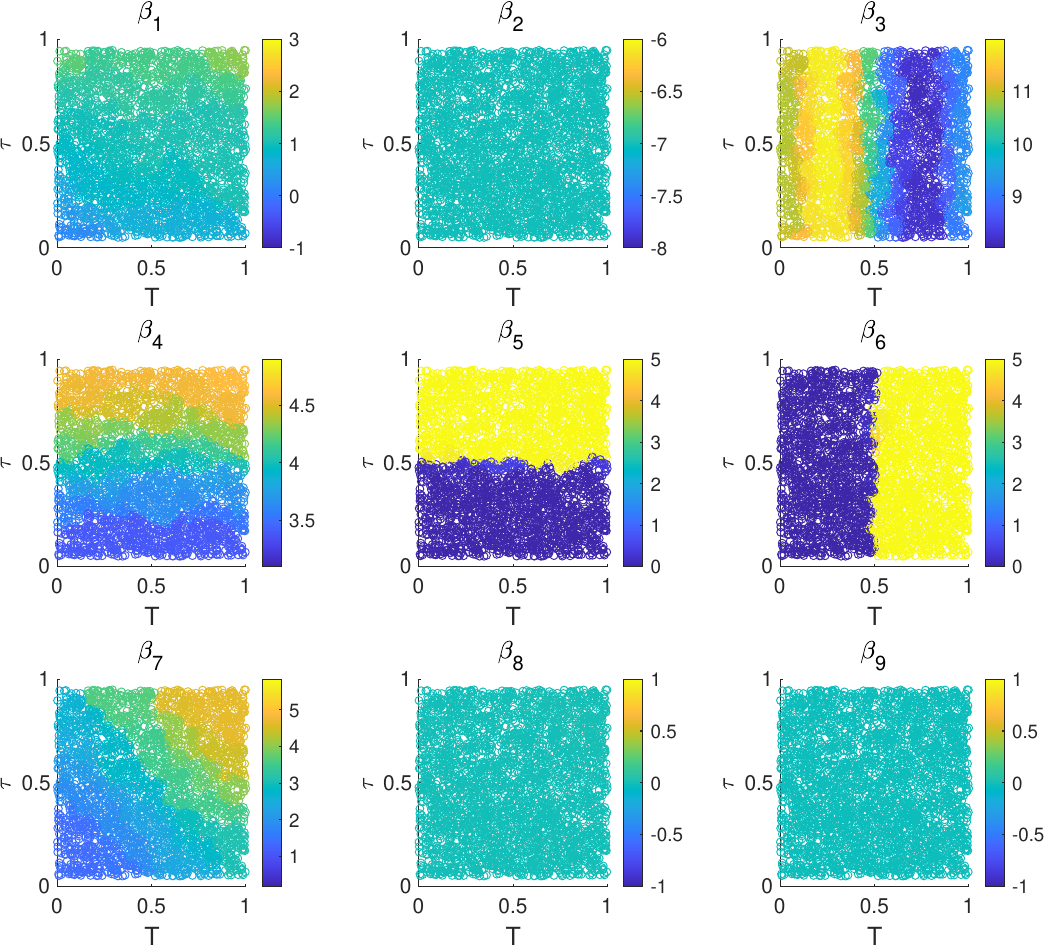}\\
\caption{Estimated coefficient function $\hat \beta_j(t,\tau)$ for $j=1,\ldots, 9$, produced from a particular simulation when  $d=9$.}
\end{figure}

 \begin{figure}[t] \label{figrplot33}
\centering
\includegraphics[width=0.7\textwidth, height=7cm]{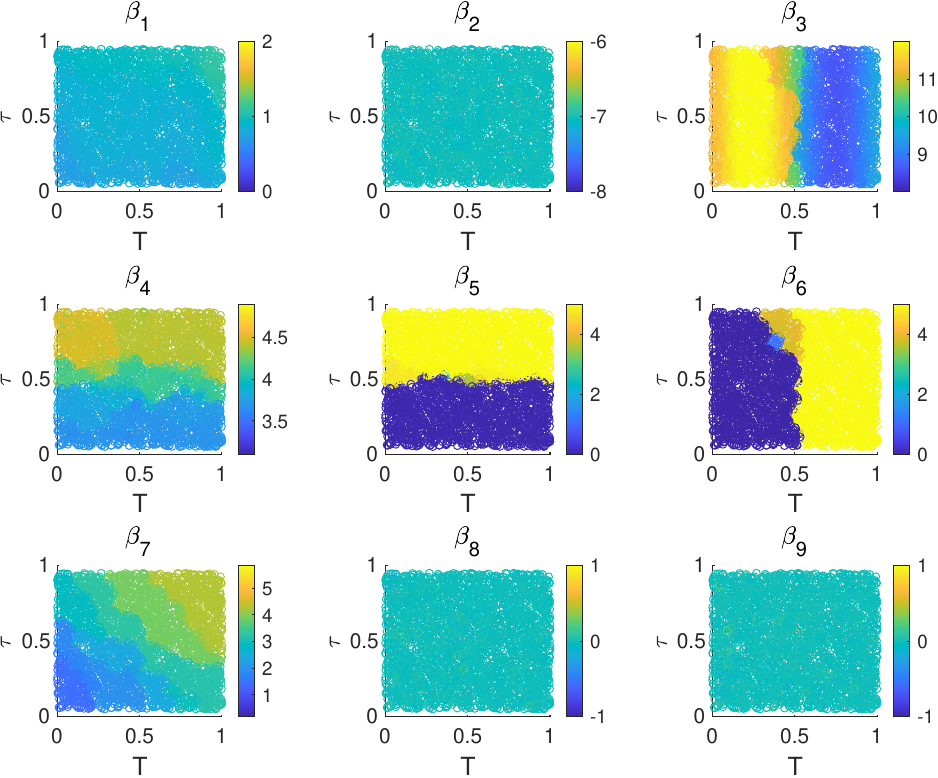}\\
\caption{Estimated coefficient function $\hat \beta_j(t,\tau)$ for $j=1,\ldots, 9$, produced from a particular simulation    with $d =25$.}
\end{figure}  

 \begin{figure}[b] \label{figrplot44}
 \centering
\includegraphics[width=0.7\textwidth, height=7cm]{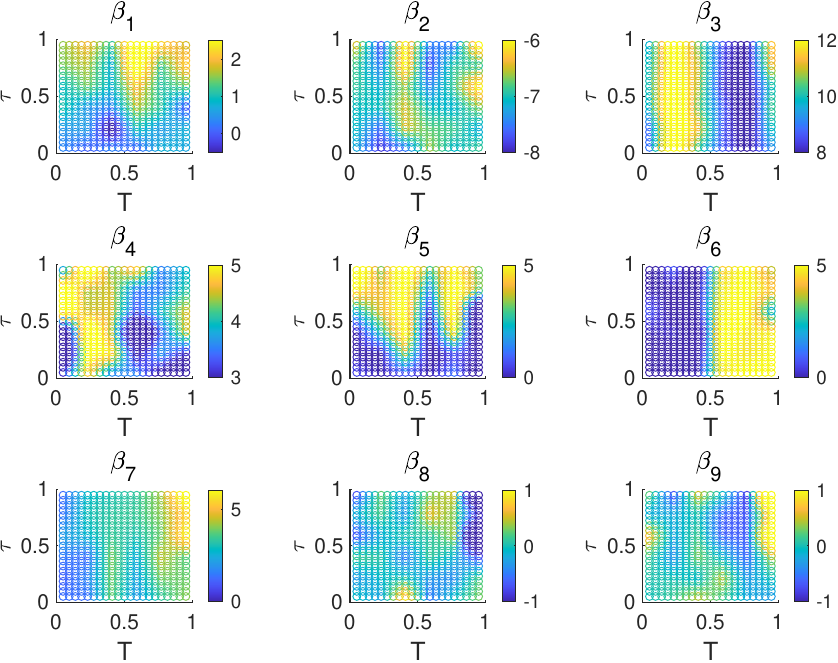}\\
\caption{An example of the estimated coefficient function of the B-spline methods  with   $d = 9$.}
\end{figure}

 \begin{figure}[t] \label{figrplot55}
 \centering
\includegraphics[width=0.7\textwidth, height=8cm]{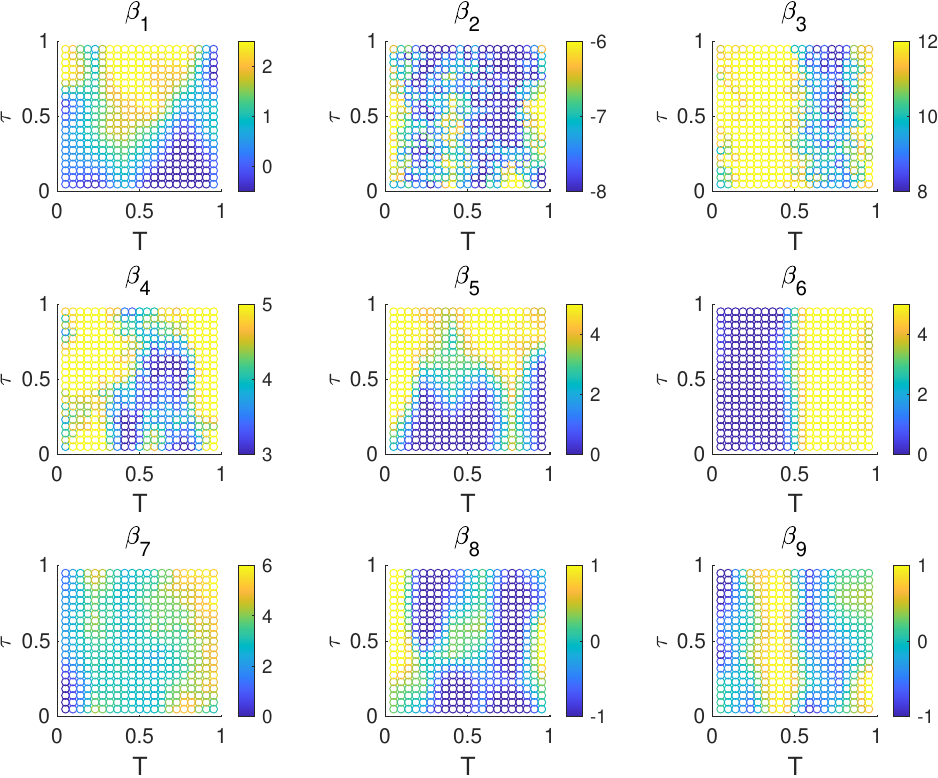}\\
\caption{An example of the estimated coefficient function of the B-spline methods with  $d =25$.}
\end{figure}

\section{Simulation Studies}
In this section, we consider simulation studies to illustrate the performance of RQF. We set $K=5$ { for sufficient information and efficient computation,} as suggested in Padilla et al. \cite{PSCW2020} {and Ye and Padilla \cite{SSYOHMP2021}.}
To examine the performance of RQF in capturing  patterns under various scenarios, we design the case in which the underlying VC coefficients have different clustered or varying patterns. 
For comparisons, we  consider the sieve estimation method using B-spline (Bspline), which is a common nonparametric approach. Specifically, Bspline approximates $\beta_j(t,\tau)$ by using bivariate B-spline functions $\Pi(\tau,t) = [\pi_1(\tau,t),\ldots, \pi_{(m_n+2)^2}(\tau,t)]^\top$, where $\pi_k(\tau,t)$ is the product of two normalized B-spline basis functions of order $2$ with $m_n$ quasi-uniform knots over the region $\Delta$ and $T$, i.e., $\beta_j(\tau,t)  \approx B_j \Pi(\tau,t)$, where $B_j$ is estimated via the composite quantile regression framework \cite{SPEL2021}  { and $m_n$ is chosen using the Bayesian information criterion (BIC).\cite{SPEL2021}} 


We consider the following varying random coefficient model: 
\begin{eqnarray*}
Y_i &=& 1 - 7 X_{i2}  + (10 + 2 \sin(2\pi t_i)) X_{i3} + (3+2U_i) X_{i4} + 5 \times 1\{U_i>0.5\} X_{i5} + 5 \times 1\{U_i>0.5\} X_{i6} \\
&&+  3(U_i + t_i) X_{i7},
\end{eqnarray*}
where  $d$ is the number of covariates, 
 $t_i, X_{i3}, X_{i6}, X_{i7} \sim \mbox{uniform}(0,1)$, $X_{i2},  X_{i8},\ldots, X_{id} \sim N(0,1)$, $X_{i4}, X_{i5}$ are from the Bernoulli distribution with probability $1/2$, and  $U_i \sim \mbox{uniform}(0,1)$ is introduced to consider a random coefficient model.
Accordingly, the underlying quantile coefficient functions, given the index $t$ and the quantile level $\tau$,  are
\begin{eqnarray*}
&& \beta_1(t,\tau)=1,\ \beta_2(t,\tau) = -7,\ \beta_3(t,\tau) = 10 + 2\sin(2\pi t),\ \beta_4(t,\tau) = 3+ 2\tau,  \\
&& \beta_5(t,\tau) = 5 \times 1\{\tau>0.5\},\ \beta_6(t,\tau) = 5 \times 1\{t>0.5\},\ \beta_7(t,\tau) = 3t + 3\tau,\   \beta_8(t,\tau) = \cdots = \beta_d(t,\tau) = 0. 
\end{eqnarray*}
 We use $n =5,000$   and $d \in \{9,25\}$ in the implementation, and the quantile levels $\tau_i$ are i.i.d. generated from [0.05, 0.95].

Figure 1 depicts the underlying coefficient functions $\beta_j(t,\tau)$ for $j=1,\ldots, 9$.  
We  observe that $\beta_3(t,\tau),\beta_4(t,\tau),$ and $\beta_7(t,\tau)$ have smoothly varying patterns; 
$\beta_5(t,\tau)$ and  $\beta_6(t,\tau)$ are clustered with respect to quantile $\tau$ and time $t$;  $\beta_1(t,\tau)$ and $\beta_2(t,\tau)$ are positive and negative, constant value, respectively; and $\beta_8(t,\tau)$ and $\beta_9(t,\tau)$ are zeros. Figures 2 and 3 are the estimated coefficient function $\hat{\beta}_j(t,\tau)$ derived by RQF produced from a particular simulation when  $d=9$ and $d=25$, respectively.
Figures 4 and 5 are the estimated coefficient function derived by Bspline from a particular simulation  when     $d=9$ and $d=25$, respectively.
We can observe that the overall patterns of RQF are highly consistent with the true regression coefficients, as shown in Figure 1. 
It successfully captures the underlying clustered patterns and smoothly varying patterns in the regression coefficients and also detects the abrupt changes across the boundaries of adjacent clusters.
However, the estimates from Bspline are quite noisy, with artificially abrupt changes in coefficient values in some parts of the domain.

We further examine the performances of RQF in terms of parameter estimation   using  $d \in \{9,25\}$. As a performance measure, we consider the mean-squared error of estimation (MSE), defined as
\[
\mbox{MSE} = \frac{1}{n} \sum_{i=1}^n \sum_{j=1}^p \{\hat \beta_j(t_i, \tau_i) - \beta_j^o(t_i,\tau_i)\}^2.
\]
{
For an index $j \in \{1,2,5,6,8,9\}$, each of the underlying coefficient functions $\beta^o_j$'s has clustered structures. Specifically, $\beta^o_1$, $\beta^o_2, \beta^o_8, \beta^o_9$ have single values, i.e., has only one cluster, and $\beta^o_5, \beta^o_6$ have two subclusters. On the other hand, an index $j \in \{3,4,7\}$ has smoothly varying structures.
For an index $j \in \{1,2,5,6,8,9\}$ involving clustered structures, we  measure the structural consistency performance  when $d \in \{9,25\}$. 
Let $S_j =\{(i,i'):\ \beta_j^o(t_i,\tau_i) \neq \beta_j^o(t_{i'},\tau_{i'})\}$ and $\hat S_j =\{(i,i'):\ \hat\beta_j(t_i,\tau_i) \neq \hat\beta_j(t_{i'},\tau_{i'})\}$.
If $|S_j| \neq 0$, i.e., $j \in \{5,6\}$, we consider the Precision and Recall for each $j$, defined as
\[
\mbox{Precision} =  |S_j \cap \hat{S}_j|/ |\hat{S}_j|\quad\mbox{and}\quad \mbox{Recall}=|S_j \cap \hat{S}_j|/ |S_j|.
\]
On the other hand, if $|S_j|=0$, i.e., $j \in \{1,2,8,9\}$,} we consider true negative rate (TNR) and negative predictive value (NPV), defined as
\[
\mbox{TNR} = |S_j^c \cap \hat{S}_j^c| /|S_j^c|\quad\mbox{and}\quad \mbox{NPV} = |S_j^c \cap \hat{S}_j^c| /|\hat S_j^c|
\]

\begin{table}[!h] \centering 
  \caption{Mean precision, recall, TNR, and NPV for RQF  and Bspline over 100 simulations,  under $d \in \{9,25\}$.}
  \label{tabs0} 
\begin{tabular}{llllll} \hline
\multirow{2}{*}{\bf Method} & \multirow{2}{*}{\bf Coefficient} & \multicolumn{2}{c}{ $d =9$ } &  \multicolumn{2}{c}{$d =25$} \\ \cline{3-6}
& & {\bf Precision}& {\bf Recall} & {\bf Precision}& {\bf Recall}\\
\toprule
\multirow{2}{*}{\bf RQF} & $\beta_5$ & 0.946 & 0.952 & 0.923 & 0.918\\
& $\beta_6$ & 0.952 & 0.953 & 0.921 & 0.913\\\hline
\multirow{2}{*}{\bf {Bspline}} & $\beta_5$ & 0.504 & 0.991 & 0.521 & 0.971\\
& $\beta_6$ & 0.510 & 0.994 & 0.531 & 0.981\\
\midrule
& & {\bf TNR}& {\bf NPV} & {\bf TNR}& {\bf NPV}\\
\midrule
\multirow{3}{*}{\bf RQF} 
& $\beta_1$ & 0.996 & 0.995 & 0.945 & 0.952 \\
& $\beta_2$ & 0.991 & 0.994 & 0.939 & 0.941 \\
& $\beta_8$ & 0.985 & 0.983 & 0.941 & 0.938\\
& $\beta_9$ & 0.989 & 0.990 & 0.950 & 0.941\\\hline 
\multirow{3}{*}{\bf {Bspline}} 
& $\beta_1$ & 0.009 & 1.000 & 0.011 & 1.000 \\
& $\beta_2$ & 0.012 & 1.000 & 0.015 & 1.000 \\
& $\beta_8$ & 0.018 & 1.000 & 0.023 & 1.000\\
& $\beta_9$ & 0.012 & 1.000 & 0.024 & 1.000\\
\bottomrule
\end{tabular} 
\end{table}


As shown in Table \ref{tabs0}, all the Precision, Recall, TNR, and NPV values {for RQF} are close to 1, which implies that RQF has a high structural consistency for each covariate $j$. {Even with the large $d$, RQF presents robust results. See  Section B of the Supplementary material for detailed results.}
On the other hand, Bspline demonstrates poor performance in  TNR and Precision  because Bspline  tends to yield more  false positives. This implies that Bspline does not capture the underlying structures of the model well.
Regarding the inferior performance of the Bspline method compared to our proposed KNN-based Lasso method, we  note that the differences in performance mainly stem from the fused Lasso term in our proposed method. This term shrinks differences between neighboring points $(t_i, \tau_i)$ and $(t_j, \tau_j)$ if they are close. This shrinkage effect cannot be easily implemented in the Bspline method and, to our knowledge, is not considered in the existing literature.

{We also investigate the sensitivity of the proposed method with respect to the choice of quantile levels $\tau_i$'s. Using the same $t_i$'s as used in $\hat\beta$,
we consider $500$ different choices of $\tau_i$'s generated from $\mbox{uniform}(0.05,0.95)$ and obtain estimates $\tilde \beta^{(b)}_j$'s for $b=1,\ldots,500$ to compare  with $\hat\beta_j$   when  $d =9$ and $d=25$, respectively. 
Then, we record the mean squared difference (MSD) between $\hat \beta_j$ and $\tilde \beta^{(b)}_j$ at the 9,000 fixed points $(t_k,\tau_l)$, where $t_k=0.01k$ for $k=1,\ldots, 100$ and $\tau_l=0.05+0.01l$ for $l=1,\ldots, 90$, defined as
\[
\mbox{MSD}(b) = \frac{1}{9000p} \sum_{k=1}^{100} \sum_{l=1}^{90} \sum_{j=1}^p \{\hat \beta_j(t_k, \tau_l) - \tilde \beta^{(b)}_j(t_k,\tau_l)\}^2,
\]
where $\hat\beta_j(t_k,\tau_l)$'s and $\tilde\beta_j(t_k,\tau_l)$'s are computed  by \eqref{eqneighbor}.
Figure B2 in the Supplementary material shows the boxplots of MSD  obtained from $500$ different quantile choices. We observe that most MSD values are close to 0,  which implies that the obtained coefficients are not highly sensitive to the choices of $\tau_i$'s. }


Let $\hat\beta^{(K)}_j$ be the proposed estimate using $K$ nearest neighbors in the estimation. 
 { To perform the sensitivity analyses of the proposed method for the choice of $K$, we compute the MSD between 
$\hat\beta^{(K)}_j$ and $\hat\beta^{(\tilde K)}_j$  at the  fixed points $(t_k,\tau_l)$, where $t_k=0.01k$ ($k=1,\ldots, 100$) and $\tau_l=0.01l+0.05$ ($l=1,\ldots, 90$).
That is,
\[
\mbox{MSD}(K,\tilde K) = \frac{1}{9000p} \sum_{k=1}^{100} \sum_{l=1}^{90} \sum_{j=1}^p \{\hat \beta^{(K)}_j(t_k, \tau_l) - \hat \beta^{(\tilde K)}_j(t_k,\tau_l)\}^2.
\]
Figure B3 in  the Supplementary material  presents the heatmap of $\mbox{MSD}(K,\tilde K)$ with $K,\tilde K \in \{2,\ldots, 12\}$ when  $d =9$ and $d=25$. The proposed method does not seem to be heavily impacted by  the  choice of $K$.} 

Another alternative is to consider multiple predetermined quantile levels $\tau$ for each $t_i$ instead of a single, randomly chosen quantile level. Specifically, to demonstrate it in our simulation, we used a setting with $d=9$ and selected a predetermined grid of quantile levels $\tau$ for each $t_i$. This resulted in $(t_i, \tau_1),\ldots, (t_i,\tau_{19})$ for each $i$, where $\tau_k = 0.05 k$ for $k=1,\ldots, 19$. As a result, a total of $n \times 19 = 5000 \times 19 = 95,000$ points were used to generate the 2D plot depicted in Figure B4 of the Supplementary material.
It is  worth noting that using multiple predetermined quantile levels for each $t_i$ necessitates the estimation of more parameters, thereby increasing computational time compared to the proposed method that uses randomly selected quantile levels. This latter approach employs only a single, randomly chosen quantile level, denoted as $\tau_i$, for each $t_i$, but still manages to yield similar patterns in the 2D plot.
However, the advantage of predetermined quantile levels is their ability to focus on specific quantile regions. If there is  particular interest in these regions, more quantile levels can be preselected specifically for those areas.

\section{Empirical illustrations} \label{secBMI}
\subsection{Time-varying and heteroscedastic effects of risk factors on BMI}

 \begin{figure}[b] \label{fig_BIC}
 \centering
\includegraphics[width=1\textwidth, height=14cm]{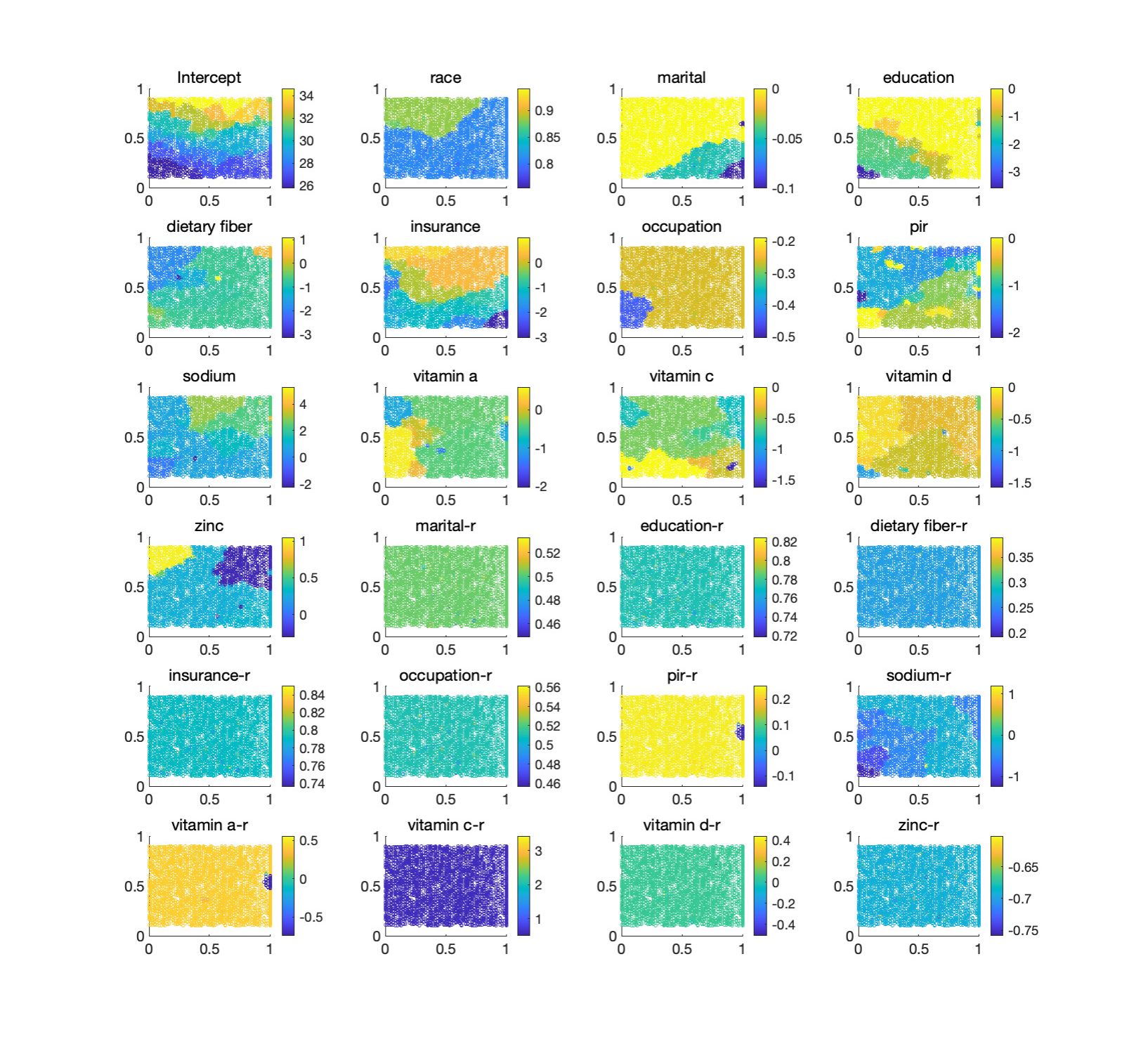}\\
\caption{{The estimated} quantile coefficient functions using the BMI data. The $x$-axis and $y$-axis represent
the {age} index $T \in (0,1)$ and the quantile level $\tau \in (0.1, 0.9)$, respectively.
{On top of the figures, `$x$-r'} represents an interaction between  a risk factor `$x$' and the `race.'}
\end{figure}

Body mass index (BMI), which is a measure of body fat based on height and weight, has been shown to be associated with many health status indicators. \cite{weisell2002body,dobbelsteyn2001comparative,vargas2007relationship}  While there has been much interest in developing statistical methods for measuring time-varying effects of risk factors on BMI, there have been few studies from the quantile perspective. Moreover, classical local quantile approaches are not well-suited for visualizing the relationship between BMI and covariates as age-quantile functions, unlike regional quantile approaches.

In this section, we utilized the RQF to investigate the relationship between risk factors and different sublevels of the BMI distribution among women, and to examine the potential variation of this relationship across age. Furthermore, we explored the health disparity between non-Hispanic Whites (NHW) and non-Hispanic Blacks (NHB) by examining the interaction between risk factors and racial groups. The data for our analysis were obtained from the 2011-2018 National Health and Nutrition Examination Survey (NHANES) dataset, and we considered 12 covariates identified in the literature as potential risk factors for BMI. These covariates include physical condition factors such as dietary fiber, sodium level, vitamin A, vitamin C, vitamin D, and zinc; socioeconomic factors such as education (1 if college, 0 otherwise), occupation (1 if yes and 0 if no), insurance (1 if yes and 0 if no), and poverty-income ratio (PIR) which ranges between 0 (lowest income level) and 1 (highest income level); and demographic information such as marital status (1 if married, 0 otherwise) and race (1 if NHB and 0 if NHW). After removing subjects with missing variables, a total of 4,119 women were available for our analysis, with 2,664 NHW and 1,455 NHB.

 The following varying-coefficient (VC) model was used to fit the data:
\begin{equation}
    \label{eq:1}
y_i = \beta_{0}(T,\tau) + \sum_{j=1}^{12} x_{j,i} \beta_{j}(T,\tau) + \sum_{j=2}^{12} x_{1,i} x_{j,i} \eta_{j}(T,\tau)+ \epsilon_i(\tau),
\end{equation}
where $y_i$ represents the BMI, $T$ represents age, and $x_{1,i}$ represents the race variable, while $\epsilon_i(\tau)$'s are independent errors satisfying $P(\epsilon_i(\tau) \le 0 \mid x_i, T_i) = \tau$. Here, $\beta_{j}(T,\tau)$ represents a quantile coefficient function for the $j$-th explanatory variable, and $\eta_{j}(T,\tau)$ represents a quantile coefficient function for the interaction between the $j$-th explanatory variable and the race variable.

To re-scale the age variable, we set $0 \le T \le 1$, where $0$ corresponds to 20 years old and $1$ corresponds to 80 years old or more. We normalized each continuous variable such that its mean and standard deviation were $0$ and $1$, respectively. To  choose $\lambda$, we used BIC as described in Section \ref{sec2vc}.

Figure 6 displays the estimated functional coefficients, where the $x$-axis and $y$-axis represent the age index $T \in (0,1)$ and the BMI quantile level $\tau \in (0.1, 0.9)$, respectively. The two-dimensional graphs describe how the coefficients vary with age and the quantile level.
Overall, most variables, except sodium level, zinc, and race, showed a negative association with BMI. The effect of education on BMI was more pronounced in the younger age group compared to the older group. In contrast, having insurance had a more constant effect across different ages, while a heteroscedastic association was observed over different BMI quantiles. For example, having insurance tended to be associated with attenuated BMI for individuals with medium or lower BMI across all ages. A higher income (i.e., higher PIR) was associated with a lower BMI, and as an individual gets older, the effect seems to be stronger at the upper level of the BMI distribution. NHB women appeared to have a higher BMI than NHW women across all ages.

The disparity in BMI between NHW and NHB women  can be  inferred by observing the interaction between race and risk factors. The positive effect sizes observed in the interaction plots suggest that the impact of risk factors on BMI levels is greater in NHB women than in NHW women. Although most interactions did not show a heteroscedastic effect, the interaction between sodium and race appeared to vary with age.

\subsection{Time-varying and heteroscedastic effects of risk factors on LDL cholesterol}

 \begin{figure}[t] \label{fig_BIC2}
 \centering
\includegraphics[width=0.9\textwidth, height=14cm]{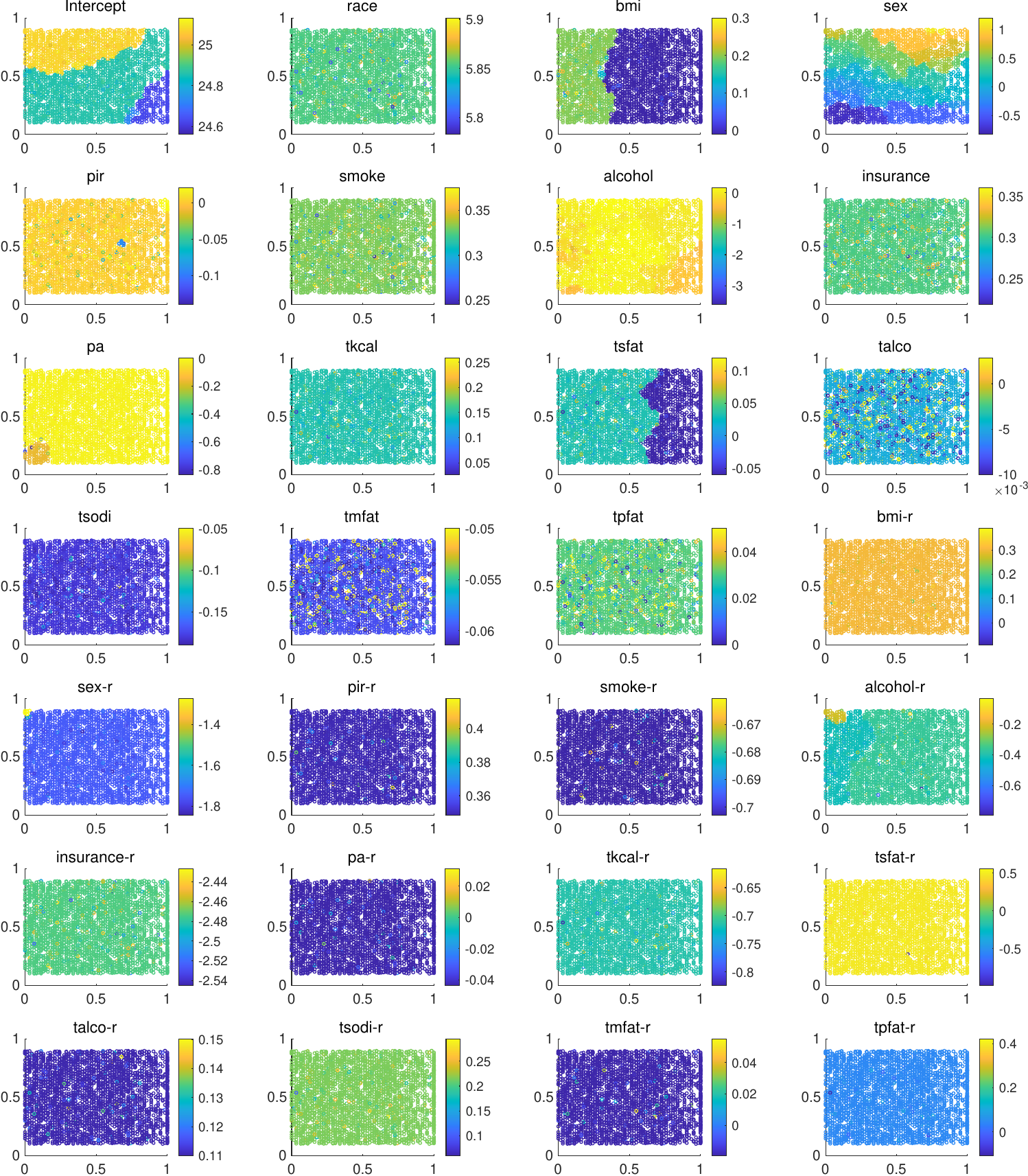}\\
\caption{The estimated quantile coefficient functions using the LDL data. The $x$-axis and $y$-axis represent
the age index $T \in (0,1)$ and the quantile level $\tau \in (0.1, 0.9)$, respectively.
On top of the figures, `$x$-r' represents an interaction between the risk factor `$x$' and the `race.'}
\end{figure}

Low-density lipoprotein (LDL) cholesterol is associated with various health problems such as cardiovascular diseases and stroke in adults. \citep{amarenco2020comparison, zeljkovic2010ldl, valdes2019relative}  In this study, we use the RQF framework to model the association between risk factors and LDL cholesterol levels (measured in mg/dL), where LDL is the health outcome of interest. Previous research has shown that LDL cholesterol levels tend to be associated with age. For instance, Ferrara et al. \cite{ferrara1997total} observed that LDL tends to increase in younger and middle-aged adults and decrease with people who are older than 65 years.

Covariates included in our model were race (NHW and NHB), PIR (0-5), sex (1 if female and 0 if male), age (18-80+), alcohol (average number of alcoholic drinks/day during the past 12 months), smoking status (1 if smoked at least 100 cigarettes in life, 0 otherwise), physical activity (continuous), BMI (continuous), energy intake (tkcal), total saturated fatty acids (tsfat, gm), alcohol intake (talco, gm), sodium intake (tsodi, mg), total monounsaturated fatty acids (tmfat, gm), and total polyunsaturated fatty acids (tpfat, gm). In addition, we included the interaction term between race and each of these covariates. The LDL data were also obtained from the 2011-2018 NHANES, and a total of 3,386 subjects were included in the analysis after excluding those with missing variables.

Similar to the BMI study, we considered the following VC model for LDL cholesterol:
\begin{equation}
    \label{eq:2}
y_i = \beta_{0}(T,\tau) + \sum_{j=1}^{14} x_{j,i} \beta_{j}(T,\tau) + \sum_{j=2}^{14} x_{1,i} x_{j,i} \eta_{j}(T,\tau)+ \epsilon_i(\tau),
\end{equation}
where $y_i$ is an LDL cholesterol level, $T$ is  age,  $x_{1,i}$ represents race,
and $\epsilon_i(\tau)$'s are independent errors satisfying 
$P(\epsilon_i(\tau) \le 0 \mid x_i, T_i) = \tau$.
Here, $\beta_{j}(T,\tau)$ represents a quantile varying-coefficient function for the $j$-th explanatory variable, whereas $\eta_{j}(T,\tau)$ represents a  quantile coefficient function for the interaction between the $j$-th covariate and a race variable.
We re-scaled the age variable to be $0 \le T \le 1$, where $0$ corresponds to 18 years old and $1$ to 80$+$ years old.

The results are presented in Figure 7. Most covariates, including race, smoking, insurance, talco, and tsodi, showed homogeneous associations with LDL levels. NHB had higher LDL cholesterol levels than NHW across all quantiles and ages. On the other hand, heteroscedastic effects were observed in BMI, sex, PA, and tsfat. For BMI, its impact on LDL varied with age, having a greater impact on younger individuals than older ones. The effect size of sex on LDL increased with higher LDL levels. Figure 7 also suggests that interactions between race and other risk factors, such as BMI, sex, PIR, smoking, alcohol, insurance, physical activity, energy intake (tkcal), alcohol, tsfat, and tsodi, are prevalent. For example, higher BMI was associated with larger increases in LDL levels for NHB compared to NHW. Increased physical activity reduced LDL levels more for NHB than NHW, suggesting that NHB may benefit more from increased physical activity in lowering LDL cholesterol levels than NHW.

Our RQF model allows us to understand how risk factors influence health outcomes by incorporating their varying effects over age and quantiles. The findings on the impact of risk factors and their interaction with race on health outcomes of interest, namely BMI and LDL cholesterol levels, provide important insights for the field of public health and health disparity research.

We note that NHANES uses a complex, multistage probability sampling design, and therefore, to avoid biased estimation, the sampling weight should be used. However, our model was not specifically designed for population-based survey data, and thus, we did not incorporate the sampling weights in our analysis. As a result, caution is needed in interpreting our results.

\section{Conclusion}
We have shown that the proposed regional quantile regression approach in varying-coefficient models, known as RQF, can provide valuable insights into health outcome studies. The RQF method is demonstrated to yield a consistent and locally adaptive estimate in various settings. It provides consistent estimation when the number of nonconstant coefficient functions $\beta_j^o$ is bounded by the rate of $n/\log (pn)$
{ and the number of different edge-connected coefficient values in the minimum spanning trees is bounded by the rate of $n\log (pn)$, i.e.,  $s_1 \lesssim n\log (pn)$.} Additionally, our approach can capture underlying smoothly varying patterns as well as cluster structures in varying-coefficients of health risks. This is because RQF employs the fused Lasso penalty function to encourage similarity in coefficients between adjacent locations when both quantile levels and the index variable are used as distance metrics in the KNN graph.

By adapting the ADMM framework to the proposed RQF approach as shown in \eqref{eqmain}, RQF is easy to implement and each step of the algorithm can be parallelized, leading to a scalable computation model. Our simulation results demonstrate that RQF can detect underlying cluster structures and smoothly varying patterns better than standard nonparametric methods that use B-splines. Analysis of the BMI and cholesterol studies reveals heteroscedastic associations between the covariates and different quantile regions of the health outcome distribution, which cannot be captured by standard VC regression approaches. Based on our theoretical investigation, we conjecture that the RQF estimation obtained in our data analyses is consistent, and the detected cluster structures are close to the truth under some regularity conditions.

There is room for exploring further variants of the RQF approach. For instance, if variable selection is a primary concern, one could consider adding an extra Lasso penalty for dimension reduction in RQF. However, our empirical studies suggest that RQF performs variable selection naturally to some extent, as it detects both underlying cluster structures and dynamic patterns in each coefficient. Another potential future direction for RQF is allowing cluster structures across different varying-coefficients simultaneously. For example, if prior knowledge suggests that some varying-coefficients share similar varying patterns or cluster structures, this information could be utilized in the estimation procedure by adding extra penalty functions that penalize the differences of those varying-coefficients.  Furthermore, considering that KNN is a nonparametric estimation method, future studies could benefit from exploring the combination of other nonparametric methods, like B-splines, with a fused lasso penalty term. We plan to report on this work elsewhere.

\medskip

{Data Availability Statement:   The data utilized in this study were obtained from the National Health and Nutrition Examination Survey (NHANES), which is conducted by the National Center for Health Statistics (NCHS) of the Centers for Disease Control and Prevention (CDC). The NHANES data are publicly available and can be accessed through the CDC website (https://www.cdc.gov/nchs/nhanes/index.htm). The code used to generate the results in this study is available in the following Github site: https://github.com/younghhk/software/tree/master/KNN.}


\vskip 1cm

\appendix
\section{Additional notations}
For $\beta_j$ and $\theta_j$ defined in Section 2.2,
we can write 
\begin{equation}\label{eq1p}
\beta_j = 
\begin{bmatrix}
\beta^j_{G_1} \\
\vdots \\
\beta^j_{G_L}
\end{bmatrix}
=
\begin{bmatrix}
[\tilde{H}_1^{(1)}]^{-1} & & \\
 & \ddots &  \\
& & [\tilde{H}_L^{(1)}]^{-1}
\end{bmatrix}
\theta_j
=
\begin{bmatrix}
[\tilde{H}_1^{(1)}]^{-1} \theta^j_{G_1} \\
\vdots\\
[\tilde{H}_L^{(1)}]^{-1} \theta^j_{G_L} \\
\end{bmatrix}. 
\end{equation}
Define matrices  $\bar{H}$ and $\ddot{H}$ as
\[
\bar{H}=
\begin{bmatrix}
[\tilde{H}_1^{(1)}]^{-1} & & \\
 & \ddots &  \\
& & [\tilde{H}_L^{(1)}]^{-1}
\end{bmatrix}
\in \mathbb{R}^{n \times n}
, \quad
\ddot{H} 
=
\begin{bmatrix}
\tilde{H}_1^{(2)}[\tilde{H}_1^{(1)}]^{-1} & & \\
& \ddots & \\
&&\tilde{H}_L^{(2)} [\tilde{H}_L^{(1)}]^{-1}
\end{bmatrix} 
\in \mathbb{R}^{(|E|+ L-n) \times n}.
\]
{By defining the new design matrix $\tilde{X}$ as
\[
\tilde{X} = [\mbox{diag}(x_1) \bar{H}, \ldots, \mbox{diag}(x_p) \bar{H}]
:=
\begin{bmatrix}
\tilde{x}_1^\top \\
\vdots \\
\tilde{x}_n^\top
\end{bmatrix} \in \mathbb{R}^{n \times (np)},
\]
we write the model expression $\sum_{j=1}^p x_{ij} \beta_{j}(t_i, \tau_i)=\tilde{x}_i^\top \theta$ in terms of the new parameter $\theta$}, where $\mbox{diag}(x_j)$ is an $n \times n$ diagonal matrix consisting of elements in the vector {$(x_{1j}, x_{2j},\ldots, x_{nj})^\top$}. Additionally, {by simple matrix algebra}, $H\beta_j$ can be expressed as
\begin{equation}\label{Hb}
H\beta_j = 
\begin{bmatrix}
H_1 \beta^j_{G_1} \\
\vdots \\
H_L \beta^j_{G_L}
\end{bmatrix} 
= 
\begin{bmatrix}
H_1 [\tilde H_1^{(1)}]^{-1} \theta^j_{G_1}\\
\vdots \\
H_L [\tilde H_L^{(1)}]^{-1} \theta^j_{G_L}
\end{bmatrix} 
= 
\begin{bmatrix}
\theta^j_{G_1, B_1}\\
\tilde{H}_1^{(2)}[\tilde{H}_1^{(1)}]^{-1}   \theta^j_{G_1} \\
\vdots \\
\theta^j_{G_L, B_1}\\
\tilde{H}_L^{(2)}[\tilde{H}_L^{(1)}]^{-1}   \theta^j_{G_L} 
\end{bmatrix}, 
\end{equation}
where $\theta^j_{G_l, {B_1}} := [\theta^j_{G_l}]_{\{2,\ldots, n_l\}} \in \mathbb{R}^{n_l-1}$ is a subvector of $\theta^j_{G_l}$  corresponding to the minimum spanning tree in the graph $G_l$.

\section{Additional results and figures}

 \begin{figure}[!h] \label{fig_mse}
 \centering
\includegraphics[width=0.35\textwidth, height=5cm]{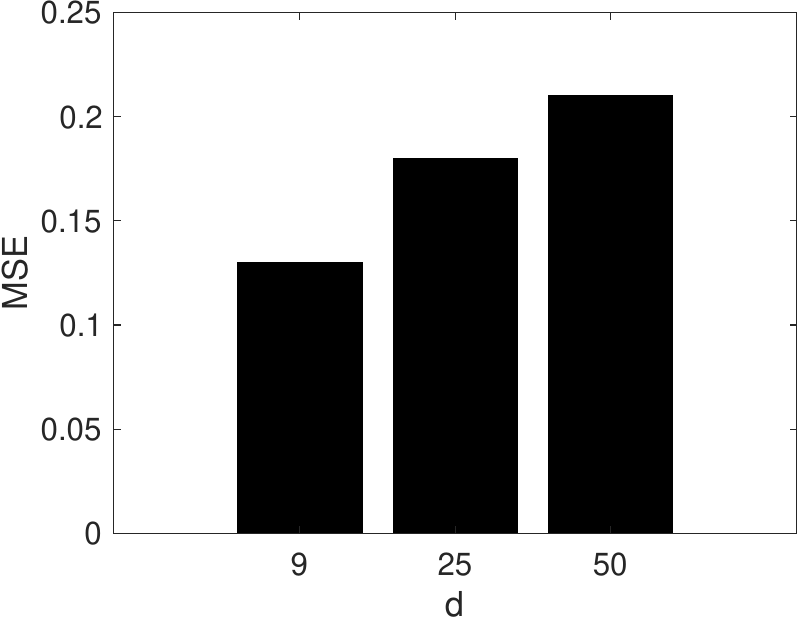}\\
\caption{The mean MSE for RQF over 100 simulations, under various noises.}
\end{figure}  

We  explored the performance of the RQF with a wide ranges of different number of covariates from  $d =9$ to $d=50$. The computed MSE is shown in Figure B1. As expected, MSE tends to deteriorate with the larger $d$. 

 \begin{figure}[!h] \label{fig_msd} 
 \centering
\includegraphics[width=0.7\textwidth, height=7cm]{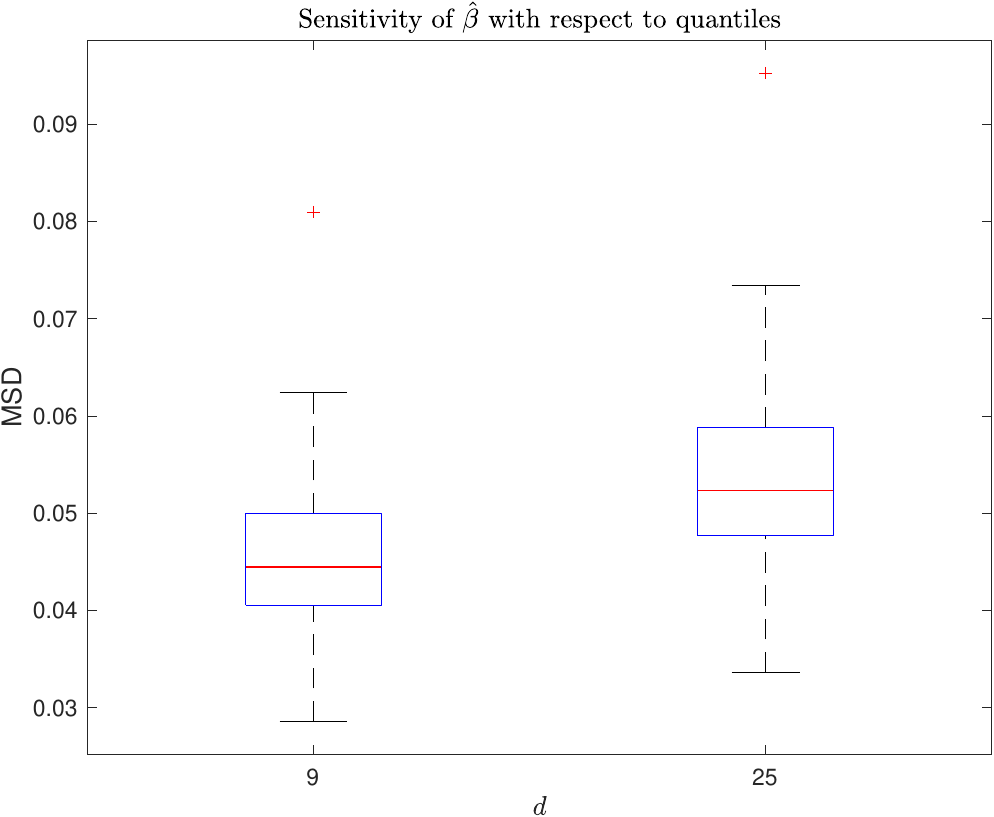}\\
\caption{{Boxplot of MSD values from $500$ different quantile choices when  $d \in \{9,25\}$.}}
\end{figure}

\begin{figure} \label{fig_msd_K} 
 \centering
  \begin{tabular}{@{}c@{}}
\includegraphics[width=0.5\textwidth, height=7cm]{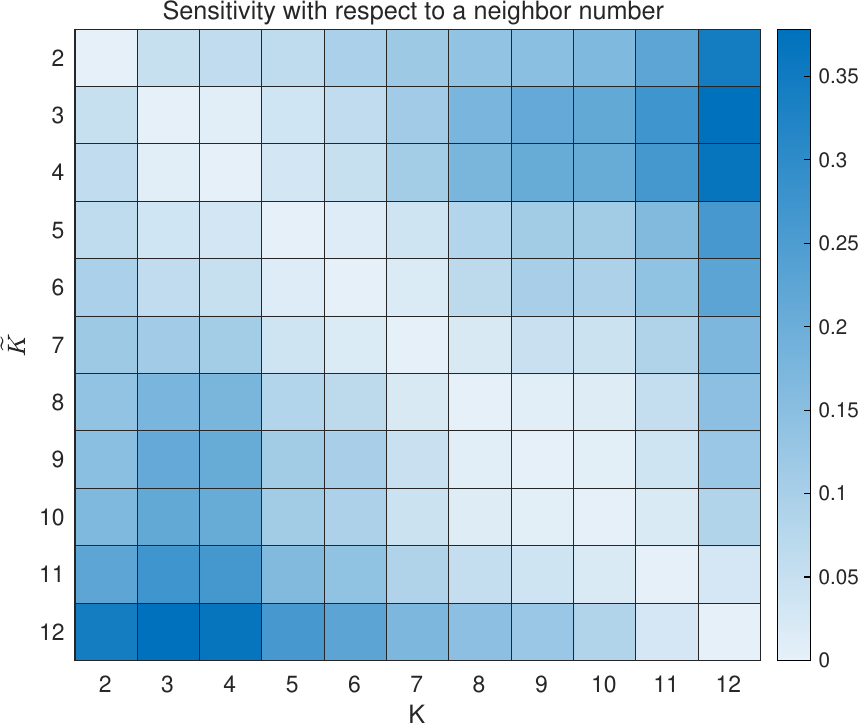}
\includegraphics[width=0.5\textwidth, height=7cm]{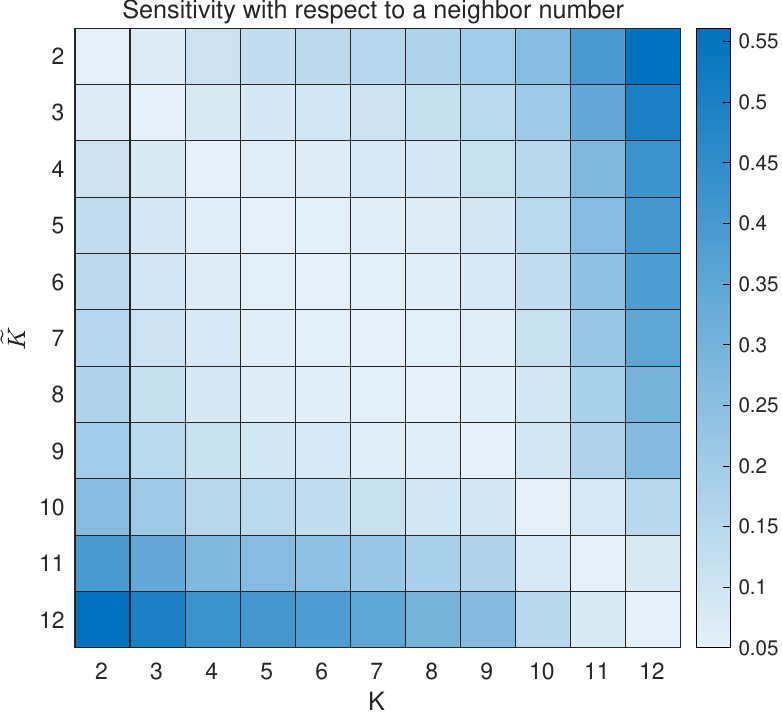}
  \end{tabular}
\caption{
 {  MSD between 
$\hat\beta^{(K)}_j$ and $\hat\beta^{(\tilde K)}_j$ with
 $K, \tilde K \in \{2,\ldots,12\}$ when   $d=9$ (left) and $d=25$ (right).}}
\end{figure}

\begin{figure} \label{figurefixed2}
\centering
\includegraphics[width=0.8\textwidth, height=8cm]{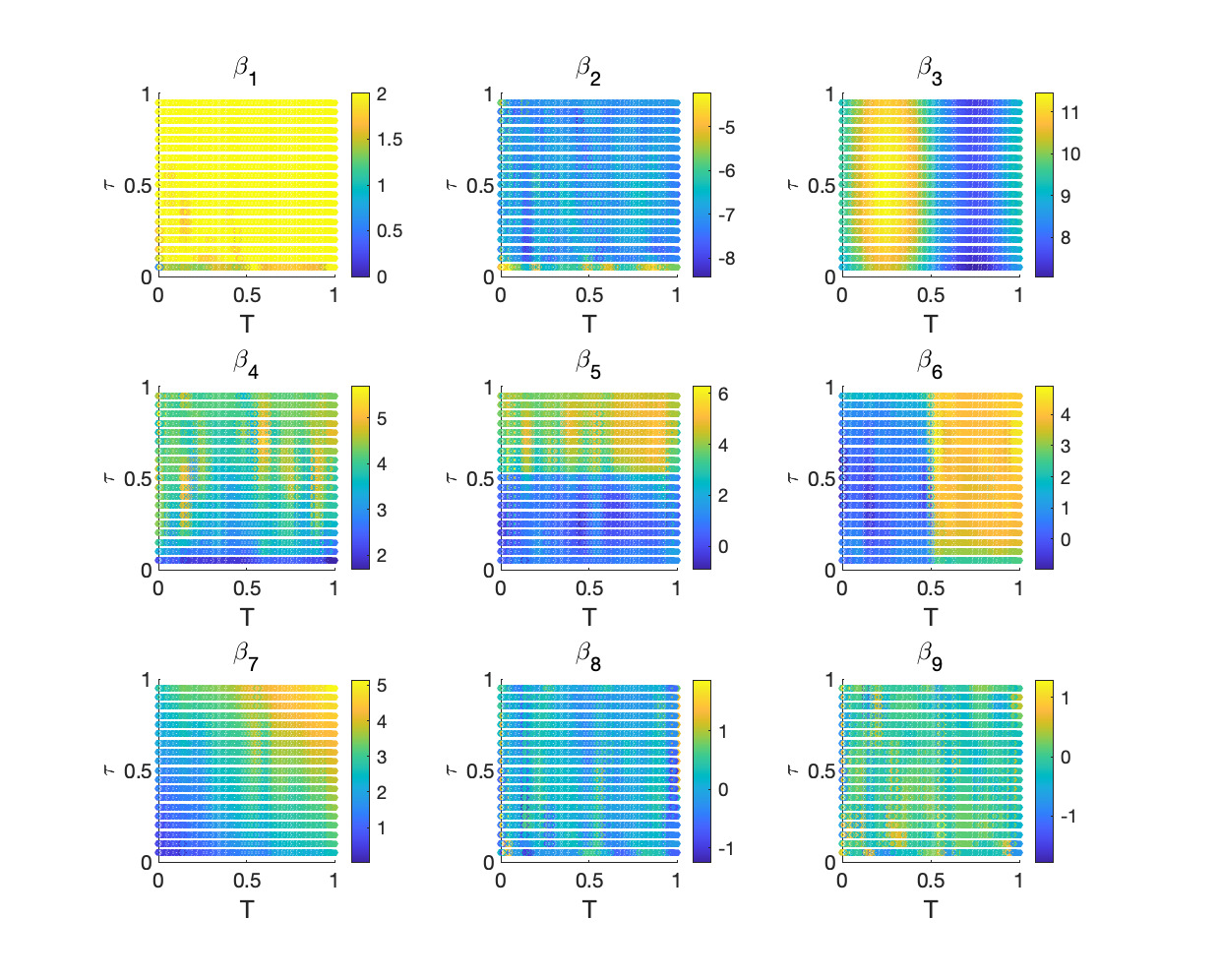}
\caption{Estimated coefficient function $\hat \beta_j(t,\tau)$ for $j=1,\ldots, 9$ when fixed quantile level is used.}
\end{figure}

\newpage
\section{Technical proofs}
In this Appendix, we include detailed proofs of theorems.

\begin{proof}[Proof of Theorem  1]
Let $\delta := \hat\theta- \theta^o = [\delta_1^\top,\ldots, \delta_p^\top]^\top$.  
For $\theta \in \mathbb{R}^{np}$, let $\hat Q(\theta)$ and   $Q(\theta)$ denote 
\[
\hat Q(\theta) := \frac{1}{n}\sum_{i=1}^n \rho_{\tau_i}\left(y_i - \tilde{x}_i^\top \theta \right),\quad 
 Q(\theta) = E[\hat Q(\theta)].
\]
Because $\hat\theta$ is the minimizer of (3), we have
\[
\hat Q(\hat\theta) + \lambda \sum_{j=1}^p \|\hat\theta_{j,B_1}\|_1 + \lambda \sum_{j=1}^p \|\ddot{H}\hat\theta_j\|_1
\le \hat Q(\theta^o) +\lambda \sum_{j=1}^p \|\theta^o_{j,B_1}\|_1 + \lambda \sum_{j=1}^p \|\ddot{H}\theta^o_j\|_1.
\]
Since $\hat Q(\cdot)$ is a convex function, we have
\begin{eqnarray*}
\hat Q(\hat\theta) &\ge& \hat Q(\theta^o)  + \langle \nabla \hat Q(\theta^o),\delta \rangle \\
&\ge& \hat Q(\theta^o) + \left\langle \sum_i \psi_{\tau_i}(y_i - \tilde{x}_i^\top \theta^o) \tilde x_i, \delta \right\rangle.
\end{eqnarray*}
Combining these inequalities with $\lambda \gtrsim \sqrt{\log(pn)/n}$, we obtain
\begin{eqnarray*}
&& \lambda  \sum_{j=1}^p \|\hat\theta_{j,B_1}\|_1 + \lambda \sum_{j=1}^p \|\ddot{H}\hat\theta_j\|_1  \\
&\le&\lambda \sum_{j=1}^p \|\theta^o_{j,B_1}\|_1 +  \lambda \sum_{j=1}^p \|\ddot{H}\theta^o_j\|_1 + \left| \left\langle \sum_i \psi_{\tau_i}(y_i - \tilde{x}_i^\top \theta^o) \tilde x_i, \delta \right\rangle \right| \\
&\le& \lambda \sum_{j=1}^p \|\theta^o_{j,B_1}\|_1 +  \lambda \sum_{j=1}^p \|\ddot{H}\theta^o_j\|_1  + 0.5 \lambda \|\delta\|_1,
\end{eqnarray*}
where the second inequality follows from $\|\sum_i \psi_{\tau_i}(y_i - \tilde{x}_i^\top \theta^o) \tilde x_i\|_{\max} \lesssim \sqrt{\log(pn)/n} \lesssim 0.5 \lambda$ using the same arguments in Theorem 1 of Belloni and Chernozhukov \cite{Belloni2011ell1}.
Hence, we have {
\begin{eqnarray*}
&& 0.5 \|\delta\|_1 
\le  \|\delta\|_1 +  \sum_{j=1}^p \|\theta^o_{j,B_1}\|_1  + \sum_{j=1}^p \|\ddot{H}\theta^o_j\|_1 -  \sum_{j=1}^p \|\hat\theta_{j,B_1}\|_1 -  \sum_{j=1}^p \|\ddot{H}\hat\theta_j\|_1 \\
&\le&  \|\delta_{B_1}\|_1 + \|\delta_{B_1^c}\|_1+ \sum_{j=1}^p \|\theta^o_{j,B_1}\|_1  + \sum_{j=1}^p \|\ddot{H}\theta^o_j\|_1 -  \sum_{j=1}^p \|\hat\theta_{j,B_1}\|_1 -  \sum_{j=1}^p \|\ddot{H}\hat\theta_j\|_1 \\
&=&  \|\hat\delta_{(S_1^U)^c}\|_1 + \|\delta_{S_1^U}\|_1 +  \|\delta_{B_1^c}\|_1+ \sum_{j=1}^p \|\theta^o_{j,S_1^U}\|_1  + \sum_{j=1}^p \|(\ddot{H}\theta^o_j)_{S_{2j}}\|_1 -  \sum_{j=1}^p \|\hat\theta_{j,(S_1^U)^c}\|_1 -\sum_{j=1}^p \|\hat\theta_{j,S_1^U}\|_1 -  \sum_{j=1}^p \|(\ddot{H}\hat\theta_j)_{S_{2j}^c}\|_1 - \sum_{j=1}^p \|(\ddot{H}\hat\theta_j)_{S_{2j}}\|_1 \\
&\le&  \|\delta_{B_1^c}\|_1 +  2\|\delta_{S_1^U}\|_1 +  \sum_{j=1}^p \|(\ddot{H}\theta^o_j)_{S_{2j}}\|_1 - \sum_{j=1}^p \|(\ddot{H}\delta_j)_{S_{2j}^c}\|_1 \\
&\le& 2\|\delta_A\|_1 + \sum_{j=1}^p \|(\ddot{H}\delta_j)_{S_{2j}}\|_1 - \sum_{j=1}^p \|(\ddot{H}\delta_j)_{S_{2j}^c}\|_1,
\end{eqnarray*}}
which implies 
\[
\|\delta_{A^c}\|_1 + \sum_{j=1}^p \|(\ddot{H}\delta_j)_{S_{2j}^c}\|_1 \le 3\left(\|\delta_{A}\|_1 + \sum_{j=1}^p \|(\ddot{H}\delta_j)_{S_{2j}}\|_1 \right),
\]
i.e., $\delta \in C(A, S_2^U)$, where the restricted set $C(A,S_2^U)$ is defined in Assumption 2.

Next, suppose that $\|\delta\| > t := C_1\sqrt{s  \log(pn)}/\sqrt{n}$ for large enough $C_1>0$.
Then, we have
\begin{eqnarray*}
0 &>&  \min_{\delta \in C(A, S_2^U), \|\delta\|_{2,n} \ge t} \hat{Q}(\theta^o + \delta) - \hat{Q}(\theta^o)  + \lambda \sum_{j=1}^p \|\hat\theta_{j,B_1}\|_1 + \lambda \sum_{j=1}^p \|\ddot{H}\hat\theta_j\|_1  - \lambda \sum_{j=1}^p \|\theta^o_{j,B_1}\|_1 - \lambda \sum_{j=1}^p \|\ddot{H}\theta^o_j\|_1.
\end{eqnarray*}
 This inequality implies 
\begin{eqnarray*}
0 &>&  \min_{\delta \in C(A, S_2^U), \|\delta\|_{2,n} = t}  \ \underbrace{\hat{Q}(\theta^o + \delta) - \hat{Q}(\theta^o)}_{I_1}  \\
 &+& \underbrace{\lambda \sum_{j=1}^p \|\hat\theta_{j,B_1}\|_1 + \lambda \sum_{j=1}^p \|\ddot{H}\hat\theta_j\|_1  - \lambda \sum_{j=1}^p \|\theta^o_{j,B_1}\|_1 - \lambda \sum_{j=1}^p \|\ddot{H}\theta^o_j\|_1}_{I_2}.
\end{eqnarray*}
Note that
\begin{eqnarray*}
|I_2| &\le&  \lambda \|\delta_{B_1}\|_1 +  \lambda \sum_{j=1}^p \|\ddot{H} \delta_j\|_1 \le  4 \lambda \left(\|\delta_{A}\|_1 + \sum_{j=1}^p \|(\ddot{H}\delta_j)_{S_{2j}}\|_1 \right)  \\
&\le& 6\lambda\sqrt{s_1}\|\delta\| +  6\lambda  \sqrt{s_2}\|(I_p \otimes \ddot{H})  \delta\| \le  6C_2\lambda\sqrt{s}\|\delta\|
\end{eqnarray*}
for some $C_2>0$.
Further, by Lemmas \ref{lemexpect} and  \ref{lem3}, we have
\[
I_1 \ge   C t^2- 
t \sqrt{s_1   \log(pn)}/\sqrt{n}.
\]
Combining these inequalities, we obtain
\[
0 > C t^2-  t\sqrt{s_1  \log(pn)}/\sqrt{n}- 6C_2 t \sqrt{s\log(pn)/n},
\]
which contradicts to the definition of $t$ with $C_1 > (1+6C_2)/C$. Thus, $\|\delta\| \le t = C_1 \sqrt{s  \log(pn)}/\sqrt{n}$, which implies
$\|\hat\theta - \theta^o\|^2 \lesssim s  \log(pn) /n$, i.e., $\|\hat\theta - \theta^o\|_n ^2 \lesssim s  \log(pn) /n^2$. This completes the proof.
\end{proof}

\vskip 0.5cm
The following lemma presents a lower bound of differences of expected values.

\begin{lemma} \label{lemexpect}
For any $\delta \in  C(A, S_2^U)$ with $\sum_i (\tilde x_i^\top \delta)^2/n \le 16q^2$,  we have for some constant $C>0$,
\[
E[\hat Q(\theta^o+ \delta)] -  E[\hat Q(\theta^o)] \ge C\|\delta\|^2.
\]
\end{lemma}	

\begin{proof}[Proof of Lemma \ref{lemexpect}]
Note that for any scalars $w$ and $v$, we have that
\[
\rho_\tau(w-v) -\rho_\tau(w) = -v(\tau-1\{w\le 0\}) + \int_0^v (1\{w \le z\} - 1\{w \le 0\}) dz.
\]
We  have
\begin{eqnarray*}
\hat Q(\theta^o+ \delta)  = \frac{1}{n}  \sum_i \rho_{\tau_{i}}(\epsilon_i(\tau_i)  - \tilde x_i^\top \delta),\quad
\hat Q(\theta^o)  = \frac{1}{n} \sum_i \rho_{\tau_{i}}(\epsilon_i(\tau_i)).
\end{eqnarray*}
Hence, we have for $\bar{z}  \in [0,z]$,
\begin{eqnarray*}
&& E[\hat Q(\theta^o+ \delta)] -  E[\hat Q(\theta^o)] \\
&=& - E\left[\frac{1}{n} \sum_{i=1}^{n} \tilde x_i^\top\delta  \left(\tau_{i} - 1\{\epsilon_i(\tau_i) \le 0 \} \right) \right] + E\left[\frac{1}{n} \sum_{i=1}^{n} \int_0^{\tilde x_i^\top \delta} \left(  1\{\epsilon_i(\tau_i) \le z\} - 1\{\epsilon_i(\tau_i)  \le 0\} \right)dz \right] \\
&=& E\left[\frac{1}{n} \sum_{i=1}^{n} \int_0^{\tilde x_i^\top \delta} \left(  1\{\epsilon_i(\tau_i) \le z\} - 1\{\epsilon_i(\tau_i)  \le 0\} \right)dz \right]\\
&=& \frac{1}{n} \sum_{i=1}^{n} \int_0^{\tilde x_i^\top \delta} \left[z f_i\left(0\right) + \frac{z^2}{2}f' (\bar{z}) 
\right]  dz \\
&\ge&  \frac{ f_i(0)}{2}   \sum_i (\tilde x_i^\top \delta)^2/n - \frac{\bar{f}'}{6} \sum_i (\tilde x_i^\top \delta)^3/n \\
&\ge&  \frac{ f_i(0)}{4}   \sum_i (\tilde x_i^\top \delta)^2/n  +  \frac{\underbar{f}}{4}   \sum_i (\tilde x_i^\top \delta)^2/n - \frac{\bar{f}'}{6} \sum_i (\tilde x_i^\top \delta)^3/n \\
&\ge& \frac{ f_i(0)}{4}   \sum_i (\tilde x_i^\top \delta)^2/n \gtrsim  C\|\delta\|^2
\end{eqnarray*}
for some constant $C>0$, where the first and second inequalities follow from Assumption  1 and the third  inequality follows from the same arguments in the proof of Lemma 4 of \cite{Belloni2011ell1} with Assumption 2 and the condition $\sum_i (\tilde x_i^\top \delta)^2/n \le 16q^2$. 
This completes the proof.
\end{proof}

Let $\|\delta\|_{2,n} :=  \sqrt{\frac{1}{n}  \sum_{i=1}^n \left\{\tilde x_i^\top  \delta \right\}^2}$.
For a random sample $Z_1,\ldots, Z_n$, let $G_n(f) := n^{-1/2} \sum_{i=1}^n (f(Z_i) - E[f(Z_i)])$.
The following Lemma \ref{lem3} presents  bounds on the empirical error over the restricted set $C(A, S_2^U)$.

\begin{lemma}\label{lem3}
Suppose that conditions of Theorem  1  hold. 
Then, with probability at least $1-1/p$,
\[
\sup_{\delta \in C(A, S_2^U):\ \|\delta\|_{2,n} \le t} \left| \hat{Q}(\theta^o + \delta) - Q(\theta^o + \delta) - \left\{\hat{Q}(\theta^o) - Q(\theta^o)  \right\}\right| \le  t\sqrt{s_1   \log(np)/n}
\]
Further, we have with probability at least $1-8/a$, for some constant $C>0$,
\[
\sup_{\delta \in \mathbb{R}^{np}:\ \delta_{(S^U_1)^c}=0,\ \|\delta\|_{2,n} \le t} \left| \hat{Q}(\theta^o + \delta) - Q(\theta^o + \delta) - \left\{\hat{Q}(\theta^o) - Q(\theta^o)  \right\}\right|  \le  12at \sqrt{C s_1 /n}.
\]
\end{lemma}

\begin{proof}
Let $\epsilon(t) := \left| \hat{Q}(\theta^o + \delta) - Q(\theta^o + \delta) - \left\{\hat{Q}(\theta^o) - Q(\theta^o)  \right\}\right| $.
For $t>0$, let {
\[
A(t) := \sqrt{n}\epsilon(t) = \sup_{\delta \in C(A, S_2^U):\ \|\delta\|_{2,n} \le t} \left|G_n\left[ \left\{ \rho_{\tau_{i}}(y_i - \tilde x_i^\top (\theta^o + \delta)) - \rho_{\tau_{i}}(y_i - \tilde x_i^\top \theta^o)\right\}   \right] \right|.
\]}
We obtain {
\begin{eqnarray*}
\mbox{Var}\left(G_n\left[\left\{ \rho_{\tau_{i}}(y_i - \tilde x_i^\top(\theta^o + \delta) - \rho_{\tau_{i}}(y_i - \tilde x_i^\top \theta^o)\right\}   \right]\right)
\le  \frac{1}{n}  \sum_{i=1}^n \left\{\tilde x_i^\top  \delta \right\}^2 = \|\delta\|_{2,n}^2  \le t^2.
\end{eqnarray*}}
Application of the symmetrization lemma in \cite{Vaart19962} yields
$
P(A(t)  \ge M) \le {2P(A^o(t) \ge M/4)}/ \left\{1- t^2 / M^2\right\},
$
where $A^o(t)$  is the symmetrized version of $A(t)$, i.e., {
\[
A^o(t) = \sup_{\delta \in C(A, S_2^U):\ \|\delta\|_{2,n} \le t} \left|G_n\left[ V_i  \left\{ \rho_{\tau_{i}}(y_i -\tilde x_i^\top(\theta^o+ \delta) - \rho_{\tau_{i}}(y_i - \tilde x_i^\top \theta^o)\right\}   \right] \right|,
\]}
where  $V_1,\ldots, V_n$ are i.i.d. Rademacher random variables.
We can write
\[
 \rho_{\tau_{i}}(y_i - \tilde x_i^\top(\theta^o + \delta)) - \rho_{\tau_{i}}(y_i -\tilde x_i^\top \theta^o) =  \tau_{i}  \tilde x_i^\top \delta + W_i(\tau_{i}, \delta),
\]
where 
$
W_i(\tau_{i}, \delta) = 
[y_i -\tilde x_i^\top(\theta^o+ \delta)]_- -  [y_i - \tilde x_i^\top\theta^o]_-,
$
where  $u_- = 1\{u < 0\} |u|$.
Then, we have $A^o(t) \le B^o(t) + C^o(t)$, where {
\begin{eqnarray*}
B^o(t) = \sup_{\delta \in C(A, S_2^U):\  \|\delta\|_{2,n} \le t} \left|G_n \left[V_i  \tilde x_i^\top\delta\right]  \right|,\quad
C^o(t) = \sup_{\delta \in C(A, S_2^U):\  \|\delta\|_{2,n} \le t} \left|G_n \left[V_i W_i(\tau_{i}, \delta)\right]  \right|.
\end{eqnarray*}}
Using the same idea in the proof of Lemma 7.2 of  \cite{Zheng2015globally}, we have
\begin{eqnarray*}
E[\exp(\lambda B^o(t))]  &=&  E \left[\exp\left(\lambda \sup_{\delta \in C(A, S_2^U):\  \|\delta\|_{2,n} \le t} \left|G_n \left[V_i  \tilde x_i^\top \delta\right]  \right|\right)\right] \\
&\le&  2np \exp(4C\lambda^2  s_1  t^2)
\end{eqnarray*}
and 
\begin{eqnarray*}
E[\exp(\lambda C^o(t))]  &=&  E \left[\exp\left(\lambda \sup_{\delta \in C(A, S_2^U):\  \|\delta\|_{2,n} \le t} \left|G_n \left[V_i W_i(\tau_{i}, \delta)\right]  \right|\right)\right] \\
&\le&  2np \exp(4C\lambda^2  s_1  t^2)
\end{eqnarray*}
for some constant $C>0$.
 Thus, we have with probability at least $1-1/p$, we have
$|B^o(t)|, |C^o(t)| \le C \sqrt{s_1  \log(np)}t$.
Hence, we have 
$
\epsilon(t) \le  Ct\sqrt{s_1  \log(np)}/\sqrt{n}.
$
This gives the first inequality.

Next, we consider the second case where  $\delta_{(S^U_1)^c} = 0$. 
We have {
\begin{eqnarray*}
E[(B^o(t))^2] &=& E\left[ \sup_{\delta_{(S^U_1)^c}=0,  \|\delta\|_{2,n} \le t} \left|\sum_{i=1}^n   V_i \tilde x_i^\top \delta / \sqrt{n} \right|^2 \right] = E\left[ \sup_{\delta_{(S^U_1)^c}=0,\  \|\delta\|_{2,n} \le t} \left|  \delta_{S^U_1}^\top \sum_{i=1}^nV_i \tilde x_{i,S^U_1} \sqrt{n} \right|^2 \right] \\
&\le& \frac{1}{n}
E\left[ \sup_{\delta_{(S^U_1)^c}=0,  \|\delta\|_{2,n} \le t} C \|\delta\|_{2,n}^2  \left\|\sum_{i=1}^n V_i \tilde x_{i,S^U_1} \right\|^2  \right]  \le \frac{t^2}{n}
E\left[ C \left\|\sum_{i=1}^n V_i \tilde x_{i,S^U_1} \right\|^2  \right] \le  C s_1  t^2,
\end{eqnarray*}}
for some $C>0$,
 By Markov inequality, 
$
P\left[B^o(t) > a \sqrt{C s_1 }t \right] \le 1/a.
$
Using the same idea in the proof of Lemma 7.3 of  \cite{Zheng2015globally}, 
$
P\left[C^o(t) > 2a \sqrt{C s_1 }t \right] \le 1/a.
$
Combining these inequalities, we obtain 
$
P\left[A(t) > 12a \sqrt{C  s_1 }t \right] \le 8/a.
$
\ignore{
\begin{eqnarray*}
E[\exp(\lambda B^o(t))] \le 2pK_n \exp(4\lambda^2  s K_n  t^2).
\end{eqnarray*}
Thus, we have with probability at least $1-1/p$,
$|B^o(t)| \le \sqrt{s K_n \log(pK_n)}t$.
Hence, we have 
\[
\epsilon(t) \le  t\sqrt{s K_n \log(pK_n)}/\sqrt{n}.
\]
}
This shows the second inequality. This completes the proof.
\end{proof}

\begin{proof}[\textbf{Proof of Theorem 2}]
For $\theta \in \mathbb{R}^{s_1}$, let
\[
\hat Q_{S}(\theta) := \frac{1}{n}\sum_{i=1}^n \rho_{\tau_i}\left(y_i - \tilde{x}_{i,S^U_1}^\top \theta \right), \quad Q_S(\theta) = E[\hat Q_S(\theta)]. 
\]
Suppose that $\hat\theta^o_{S}$ is the oracle estimator using the underlying support sets $S^U_1$ and $S^U_2$,  i.e., 
\[
\hat \theta^o_{S} := \argmin_{\theta \in \mathbb{R}^{s_1}, \ \ddot{H}_{S_{2j}^c, S_{1j}} \theta_j = 0} \hat Q_{S}(\theta) + \lambda  \|\theta\|_1 + \lambda \sum_{j=1}^p \|\ddot{H}_{S_{2j}, S_{1j}} \theta_j\|_1,
\]
where $\theta_j \in \mathbb{R}^{s_{1j}}$, $\ddot{H}_{S_{2j}, S_{1j}}$ is the submatrix of $\ddot{H} \in \mathbb{R}^{s_{2j} \times s_{1j}}$ corresponding to the  rows  in $S_{2j}$ and columns in $S_{1j}$, and $\ddot{H}_{S_{2j}^c, S_{1j}}$ is the submatrix of $\ddot{H}$ corresponding to the  rows  in $S_{2j}^c$ and columns in $S_{1j}$.
Define the block diagonal matrices $\dddot{H}_1$ and $\dddot{H}_2$  as follows:
\[
\dddot{H}_1 = 
\mbox{Block}(\ddot{H}_{S_{21}, S_{11}}, \ldots, \ddot{H}_{S_{2p}, S_{1p}}) \in \mathbb{R}^{s_2 \times s_1}, \quad
\dddot{H}_2 = 
\mbox{Block}(\ddot{H}_{S_{21}^c, S_{11}}, \ldots, \ddot{H}_{S_{2p}^c, S_{1p}}) \in \mathbb{R}^{(p|E| + pL - pn - s_2) \times s_1}
\]
Then,  the oracle estimator can be simply rewritten as 
\[
\hat \theta^o_{S} := \argmin_{\theta \in \mathbb{R}^{s_1}, \dddot{H}_2 \theta = 0} \hat Q_{S}(\theta) + \lambda  \|\theta\|_1 + \lambda \|\dddot{H}_1 \theta\|_1.
\]
Because  $\theta$ is in the null space of $\dddot{H}_2$, i.e., $\theta \in \mbox{null}(\dddot{H}_2) \subseteq \mathbb{R}^{s_1}$, we can write $\theta = T \tilde \theta$ for some matrix $T \in \mathbb{R}^{s_1 \times r}$ satisfying $T^\top T = I_r$ and $ \tilde \theta \in \mathbb{R}^r$ , where $r \le s_1$ is the dimension of the subspace $\mbox{null}(\dddot{H}_2)$.
Hence,  the oracle estimator can be rewritten as  $\hat \theta^o_{S}  = T \tilde \theta^o_S$, where
\[
\tilde \theta^o_S := \argmin_{\theta \in \mathbb{R}^{r}}  \hat Q_S(T\theta) + \lambda  \|T\theta\|_1 + \lambda \|\dddot{H}_1 T\theta\|_1,
\]
Let $ \theta^o_{S^U_1} = T \ddot \theta^o_S$.
Let $\ddot \delta_{S} := \tilde\theta^o_S-  \ddot \theta^o_S$.  
By Lemma \ref{lem3},  we can write
\[
\hat{Q}_S(T\ddot \theta^o_S+ T\ddot\delta_{S}) - \hat{Q}_S(T \ddot \theta^o_S)  =  Q_S(T \ddot \theta^o_S +T\ddot \delta_S) -  Q_S(T \ddot \theta^o_S) + O_p\left(\|T\ddot\delta_{S}\|  \sqrt{s_1/n} \right).
\]
Thus, $\ddot\delta_{S}$  is the solution of the optimization
\[ 
\min_{\delta_{S} \in \mathbb{R}^{r}}    Q_S(T \ddot \theta^o_S + T\delta_S) + O_p\left(\|T\delta_{S}\| \sqrt{s_1 /n} \right)
 + \lambda   \|T \ddot \theta^o_S + T\delta_{S} \|_1 +  \lambda \|\dddot{H}_1 T \ddot \theta^o_S + \dddot{H}_1T\delta_{S}\|_1.
\]
Then by the KKT condition, we have
\begin{eqnarray*}
&& T^\top \nabla  Q_S( T\ddot \theta^o_S +T\ddot \delta_S) + T^\top v + \lambda T^\top u + \lambda T^\top \dddot{H}_1^\top t = 0, 
\end{eqnarray*}
where $v$ is some $s_1$-dimensional vector with $\|v\| = O_p(\sqrt{s_1 /n})$, $u =\mbox{sign}(T \ddot \theta^o_S + T\ddot\delta_{S}) \in \mathbb{R}^{s_1}$,
 and  $t=  \mbox{sign}(\dddot{H}_1 T \ddot \theta^o_S + \dddot{H}_1T \ddot\delta_{S}) \in \mathbb{R}^{s_2}$, where $\mbox{sign}$ function is  applied elementwise to the components,
that is,
\begin{eqnarray} \label{fhaeil}
&& T^\top \nabla  Q_S( \hat\theta^o_S) + T^\top v + \lambda T^\top u + \lambda T^\top \dddot{H}_1^\top t = 0.  
\end{eqnarray}
Let $M= \int_0^1 \nabla^2  Q_S( \ddot \theta^o_S+ t(\tilde \theta^o_S- \ddot \theta^o_S)) dt \in \mathbb{R}^{s_1 \times s_1}$. 
Because $\nabla  Q_S(T \ddot \theta^o_S + T\ddot\delta_S)  = \nabla  Q_S(T \ddot \theta^o_S) + M T\ddot \delta_{S} = M T\ddot \delta_S$, 
we have
\[
\ddot\delta_{S} = - (T^\top M T)^{-1} \left(T^\top v+  \lambda T^\top u + \lambda T^\top \dddot{H}_1^\top t  \right).
\]
Hence, we have
\begin{equation}\label{abhildse}
\hat \theta^o_{S}  - \theta^o_{S^U_1} =  - T (T^\top M T)^{-1} \left(T^\top v +  \lambda T^\top u + \lambda T^\top \dddot{H}_1^\top t  \right).
\end{equation}

Now, we consider the following estimator
\begin{equation}\label{segafi}
\hat \theta^o_{S^U_1} := \argmin_{\theta \in \mathbb{R}^{s_1}} \hat Q_{S}(\theta) + \lambda  \|\theta\|_1 + \lambda \|\dddot{H}_1 \theta\|_1 + \lambda \|\dddot{H}_2 \theta\|_1.
\end{equation}
For any $\theta  \in \mathbb{R}^{s_1}$, we can write it as $\theta = T \theta_1 + T^\perp \theta_2$ for some $\theta_1 \in \mathbb{R}^r$ and $\theta_2 \in \mathbb{R}^{s_1-r}$, where $T^\perp \in \mathbb{R}^{s_1 \times (s_1-r)}$ is an orthogonal complement of $T$.
Hence, we can write $\hat \theta^o_{S^U_1} = T \hat\theta_1 + T^\perp \hat\theta_2$, where
\[
(\hat \theta_1, \hat\theta_2) := \argmin_{\theta_1 \in \mathbb{R}^{r}, \theta_2 \in \mathbb{R}^{s_1-r}} \hat Q_{S}(T \theta_1 + T^\perp \theta_2) + \lambda  \|T \theta_1 + T^\perp \theta_2\|_1 + \lambda \|\dddot{H}_1 T \theta_1 + \dddot{H}_1 T^\perp \theta_2\|_1 + \lambda \|  \dddot{H}_2 T^\perp \theta_1 +\dddot{H}_2 T^\perp \theta_2\|_1.
\]
Next, we will prove that $\hat\theta_2=0$, which implies $\hat\theta^o_{S^U_1}= \hat \theta^o_{S} $. By KKT condition, \eqref{fhaeil}, and $\dddot{H}_2 T^\perp \hat\theta^o_{S} =0$, $\hat\theta_2=0$ is equivalent to satisfying 
\begin{eqnarray} \label{agjildf}
&& (T^\perp)^\top \nabla  Q_S( \hat \theta^o_{S}) + (T^\perp)^\top v+ \lambda (T^\perp)^\top u + \lambda (T^\perp)^\top \dddot{H}_1^\top t  +  \lambda (T^\perp)^\top \dddot{H}_2^\top \bar{t} = 0, 
\end{eqnarray}
 for some $\bar{t} \in \mathbb{R}^{p|E| + pL - pn - s_2}$ with $\|\bar t\|_{\max} \le 1$, and   $t=  \mbox{sign}(\dddot{H}_1 \hat \theta^o_{S})$ and $u =\mbox{sign}(\hat \theta^o_{S})$.
 By $\nabla  Q_S(\hat \theta^o_{S})  = \nabla  Q_S(\theta^o_{S^U_1}) + M (\hat \theta^o_{S}-\theta^o_{S^U_1}) = M (\hat \theta^o_{S}-\theta^o_{S^U_1})$ and \eqref{abhildse},  \eqref{agjildf} is equivalent to satisfying \ignore{
\[
- (T^\perp)^\top  M 
T (T^\top M T)^{-1} \left(T^\top v +  \lambda T^\top u + \lambda T^\top \dddot{H}_1^\top t  \right) +  (T^\perp)^\top v
 + \lambda (T^\perp)^\top u + \lambda (T^\perp)^\top \dddot{H}_1^\top t +\lambda (T^\perp)^\top \dddot{H}_2^\top \bar{t} = 0,
\]}
\begin{equation}\label{eqfneial}
-  M 
T (T^\top M T)^{-1} \left(T^\top v +  \lambda T^\top u + \lambda T^\top \dddot{H}_1^\top t  \right) +   v
 + \lambda  u + \lambda \dddot{H}_1^\top t +\lambda T^\perp (T^\perp)^\top \dddot{H}_2^\top \bar{t} = 0,
\end{equation}
\ignore{ 
that is,
\begin{equation}\label{eqfneial}
\left[- (T^\perp)^\top  M T (T^\top M T)^{-1}  T^\top + (T^\perp)^\top\right] v =\lambda \left[ (T^\perp)^\top  M T (T^\top M T)^{-1} (T^\top u +  T^\top \dddot{H}_1^\top t  ) -   (T^\perp)^\top u - (T^\perp)^\top \dddot{H}_1^\top t - (T^\perp)^\top \dddot{H}_2^\top \bar t\right].
\end{equation}}
Let $\ddot{T} = M T (T^\top M T)^{-1} T^\top$. By $\lambda \asymp \sqrt{\log(pn)/n}$, $s_1 = O(n \log (pn))$, and Assumption 4, we have
 $\|v+ \lambda  u + \lambda \dddot{H}_1^\top t\|_{\max} \lesssim \lambda$ and 
 \[
 \|M 
T (T^\top M T)^{-1} \left(T^\top v +  \lambda T^\top u + \lambda T^\top \dddot{H}_1^\top t  \right)\|_{\max} \le \vertiii{\ddot{T}}_{\infty} \|v+ \lambda  u + \lambda \dddot{H}_1^\top t\|_{\max} \lesssim \lambda.
\]
Because  
\[
\|\lambda T^\perp (T^\perp)^\top \dddot{H}_2^\top \bar{t}\|_{\max} \le \lambda \vertiii{T^\perp (T^\perp)^\top \dddot{H}_2^\top}_{\infty}  \| \bar{t}\|_{\max} \le  \lambda \vertiii{\dddot{H}_2}_{1}  \| \bar{t}\|_{\max} \le \lambda  \| \bar{t}\|_{\max},
\]
there exists some $\bar t$ with  $\|\bar t\|_{\max} \le 1$ satisfying  \eqref{eqfneial}.  Thus,  $\hat \theta^o_{S} = \hat\theta^o_{S^U_1}$.

{Note that by Lemma \ref{lem3}, $\delta := \hat\theta - \theta^o$  is  the solution of the optimization
\begin{eqnarray} \label{eqmadf}
\min_{\delta \in \mathbb{R}^{np}}   Q(\theta^o+\delta) + O_p\left(\|\delta\| \sqrt{s_1 \log(np)/n} \right)
 + \lambda   \|\theta^o +\delta \|_1 + \lambda  \|\dddot{H}(\theta^o + \delta)\|_1,
\end{eqnarray}}
where $\dddot{H} =  I_p \otimes \ddot{H}$.
Next, we construct $\hat \theta^o \in \mathbb{R}^{np}$ such that $[\hat \theta^o]_{S^U_1} = \hat\theta^o_S$ and $[\hat \theta^o]_{(S^U_1)^c} =0$. To show that $\hat\theta^o - \theta^o$ is a minimizer of \eqref{eqmadf},
it is enough to show
\begin{eqnarray} \label{eqmaind}
&& \nabla Q(\hat \theta^o) + O_p\left(\sqrt{s_1  \log(np)/n} \right) + \lambda u+ \lambda \dddot{H}^\top t= 0,
\end{eqnarray}
where  $t = \mbox{sign}(\dddot{H}\hat \theta^o)$ and $u$ satisfies  $u_{S^U_1}=\mbox{sign}(\hat \theta^o_S)$, and $\|u_{(S^U_1)^c}\|_{\max} \le 1$.  

Let $\bar M= \int_0^1 \nabla^2  Q(\theta^o+ t(\hat \theta^o-  \theta^o)) dt \in \mathbb{R}^{np \times np}$.
Let $\delta := \hat\theta^o - \theta^o$.
Because $\nabla Q(\theta^o_{S^U_1}+\delta_{S^U_1})  = \nabla Q(\theta^o_{S^U_1}) + \bar M_{S^U_1,S^U_1} \delta_{S^U_1} =\bar M_{S^U_1,S^U_1}\delta_{S^U_1}$, we have by \eqref{segafi} and the KKT condition, 
\[
\delta_{S^U_1} = - \bar M_{S^U_1,S^U_1}^{-1} \left(O_p\left(\sqrt{s_1/n} \right) + \lambda \bar u + \lambda \dddot{H}_1^\top t  + \lambda \dddot{H}_2^\top \tilde t   \right),
\]
where $\bar u=\mbox{sign}(\hat \theta_{S^U_1}^o)$, $t=\mbox{sign}( \dddot{H}_1 \hat \theta_{S^U_1}^o)$, and $\tilde t=\mbox{sign}( \dddot{H}_2 \hat \theta_{S^U_1}^o)$.

Note that we have
{
\begin{eqnarray*}
 \|[\nabla Q(\hat \theta^o)]_{(S^U_1)^c}\|_{\max}  &\le&  \|[\nabla Q(\theta^o)]_{(S^U_1)^c}\|_{\max}  + \left\|\left\{\bar M(\hat\theta^o - \theta^o)\right\}_{(S^U_1)^c}\right\|_{\max} \\
 &=&  \|\bar M_{(S^U_1)^c,S^U_1} \delta_{S^U_1}\|_{\max} \\
 &=&  \left\|\bar M_{(S^U_1)^c,S^U_1} \bar M_{S^U_1,S^U_1}^{-1} \left(O_p\left(\sqrt{s_1/n} \right) + \lambda \bar u + \lambda \dddot{H}_1^\top t +  \lambda \dddot{H}_2^\top \tilde t\right)\right\|_{\max}\\
 &\lesssim&  \vertiii{\bar M_{(S^U_1)^c,S^U_1} \bar M_{S^U_1,S^U_1}^{-1}}_1 \sqrt{s_1  \log(pn)}/\sqrt{n}\\
 &\lesssim& \sqrt{s_1   \log(pn)}/\sqrt{n},
\end{eqnarray*}}
where the first inequality follows from  Assumption 4 and the second inequality follows from the order  of $\lambda$.
Thus, $\hat\theta^o - \theta^o$ satisfies \eqref{eqmaind}.
Hence, $\hat\theta^o$ is the minimizer of the original program, i.e., we have  $\hat\theta_{(S^U_1)^c}=0$ and $\dddot{H}_2 \hat\theta =0$.

Further, by the beta-min condition, 
\[
\min_{j \in S^U_1} |\hat\theta_j| \ge  \min_{j \in S^U_1} |\theta^o_j| - \|\hat\theta - \theta^o\| =  \min_{j \in S} |\theta^o_j|  - O_p\left(\sqrt{s \log(np)/n} \right) > 0.
\] 
\[
\min_j  [\dddot{H}_1 \hat\theta_{S^U_1}]_j \ge \min_j  [\dddot{H}_1(\hat\theta_{S^U_1}- \theta^o_{S^U_1})]_j - \|\dddot{H}_1 \theta^o_{S^U_1}\|
\ge   \min_j  [\dddot{H}_1(\hat\theta_{S^U_1}- \theta^o_{S^U_1})]_j - O_p\left(\sqrt{s \log(np)/n} \right) > 0.
\]
Hence, we have $\hat{S}^U_1 = S^U_1$ and $\hat S^U_2 = S^U_2$.
This completes the proof.
\ignore{
Define a matrix $M$ as:
\[
M= \int_0^1 \nabla^2 Q(\beta  + t(\hat\beta - \beta)) dt 
\] 
Then, we have
\[
M \delta =  \nabla Q(\beta+\delta)  -  \nabla Q(\beta) =  O_p\left(\sqrt{sp \log(np)/n} \right) - \lambda (s_j)_{1 \le j \le np} -   \nabla Q(\beta).
\]
Thus, we obtain
\begin{eqnarray*}
\|\delta\|_{\max} &\le&  \left\|M^{-1} \left(O_p\left(\sqrt{sp \log(np)/n} \right) - \lambda (s_j)_{1 \le j \le np} -   \nabla Q(\beta)\right)\right\|_{\max} \\
&\le& \|M^{-1}\|_{op} O_p\left(\sqrt{sp \log(np)/n} \right)  + \lambda  |\| M^{-1} \||_1 + \|\nabla Q(\beta)\|_{\max}  |\| M^{-1} \||_1 \\
&\le& O_p\left(\sqrt{sp \log(np)/n} \right).
\end{eqnarray*}}
\end{proof}


\bibliography{ref}
\end{document}